\documentclass{article}
\usepackage{ijcai26}

\usepackage{times}
\usepackage{soul}
\usepackage{url}
\usepackage[hidelinks]{hyperref}
\usepackage[utf8]{inputenc}
\usepackage[small]{caption}
\usepackage{graphicx}
\usepackage{amsmath}
\usepackage{amsthm}
\usepackage{booktabs}
\usepackage{algorithm}
\usepackage[switch]{lineno}
\usepackage{color}

\usepackage{amsmath}
\DeclareMathOperator*{\argmax}{arg\,max}

\usepackage{dblfloatfix}
\usepackage{graphicx}
\usepackage{subcaption}
\usepackage{gensymb}
\usepackage{booktabs}
\usepackage{xcolor}
\usepackage{algorithm}
\usepackage{amsfonts}
\usepackage{amsmath}
\usepackage{algpseudocode}
\usepackage{courier}
\usepackage{amsthm}

\newtheorem{example}{Example}[section]
\newtheorem{proposition}{Proposition}[section]

\newtheorem{definition}{Definition}[section]

\def\domainSet{\mathcal{C}}
\def\domainVar{\mathcal{V}^{var}}
\def\chunkStory{\domainSet_c^s}
\def\groundatoms{\mathcal{G}}
\def\parsimony{\sigma}
\def\hypothesisSet{\mathcal{H}}
\def\observationSet{\mathcal{O}}
\def\explanationSet{\mathcal{E}}
\def\logicProgram{\Pi}
\def\lattice{\mathcal{L}}
\def\sfeat{\texttt{s\_feat}}
\def\cfeat{\texttt{c\_feat}}
\def\head{\texttt{corpus\_sim}}
\def\contains{\texttt{contains}}
\def\conf{\texttt{conf}}
\def\assoc{\texttt{associated}}
\def\fixpointOperator{\Gamma}

\usepackage{tcolorbox}
\usepackage{xcolor}

\title{From We to Me: Theory Informed Narrative Shift with Abductive Reasoning}

\author{
Jaikrishna Manojkumar Patil$^{1*}$
\and
Divyagna Bavikadi$^{1}$\thanks{equal contribution}\and
Kaustuv Mukherji$^{1}$\and\\
Ashby Steward-Nolan$^1$\and
Peggy-Jean Allin$^2$\and
Tumininu Awonuga$^2$\and
Joshua Garland$^2$\\\And
Paulo Shakarian$^1$\\
\affiliations
$^1$Syracuse University\\
$^2$Arizona State University\\
\emails
\{jpatil01, dbavikad, kmukherj, astewa29\}@syr.edu,
\{pallin, osawonug, Joshua.Garland\}@asu.edu,
pashakar@syr.edu
}

\begin{document}

\maketitle

\begin{abstract}

Effective communication often relies on aligning a message with an audience's narrative and worldview. Narrative shift involves transforming text to reflect a different narrative framework while preserving its original core message—a task we demonstrate is significantly challenging for current Large Language Models (LLMs). To address this, we propose a neurosymbolic approach grounded in social science theory and abductive reasoning. Our method automatically extracts rules to abduce the specific story elements needed to guide an LLM through a consistent and targeted narrative transformation. Across multiple LLMs, abduction-guided transformed stories shifted the narrative while maintaining the fidelity with the original story. For example, with GPT-4o we outperform the zero-shot LLM baseline by $55.88\%$ for collectivistic to individualistic narrative shift while maintaining superior semantic similarity with the original stories ($40.4\%$ improvement in KL divergence). 
    For individualistic to collectivistic transformation, we achieve comparable improvements. We show similar performance across both directions for Llama-4, and Grok-4 and competitive performance for Deepseek-R1. 
    
    %
\end{abstract}

\section{Introduction}

        
        
        
        
        
        
        
        
Narratives function as cognitive constructs that shape social reality and influence how individuals interpret events~\cite{javadi2025can,li-etal-2024-search,bruner1991narrative}. Cultural orientations—particularly the distinction between individualism and collectivism systematically affect how people attribute causality, assign moral responsibility, and evaluate success or failure~\cite{choi2025individualism,markus2014culture,triandis1995individualism}. While LLMs have been applied to narrative-adjacent task such as story generation~\cite{rooein2025biased,tian2024large,pei2024swag}, narrative discourse analysis~\cite{jenner2025using,piper-bagga-2024-using}, text style transfer~\cite{tao2025cat,mukherjee2024large}
, they have not been shown to transform a narrative's cultural orientation while preserving its core semantics. This problem is non-trivial— as illustrated in Figure~\ref{fig:narrative_transform}, a zero-shot LLM prompt fails to identify and modify the key narrative marker (`\textit{all hands dig—young, old, weak, strong}') that signals collectivism. We propose a neurosymbolic abduction-guided approach grounded in social-science theory that identifies the specific story elements that require transformation. As seen in Figure~\ref{fig:narrative_transform}, our abduction-guided transformation correctly identifies the key collectivistic phrase and modifies it into `\textit{only when one determined soul digs.}' Across multiple LLMs, our approach outperforms the zero-shot baseline by up to 55.88\% in shifting narrative while maintaining up to 40.4\% better semantic fidelity with the original story. Note that we learns rules from social-science theory to ensure narrative transformations are aligned with the target corpus, rather than simply maximizing narrative orientation ratings. 

\begin{figure}[t!]
    \centering
    \fbox{
    \begin{minipage}{0.48\textwidth}
        \small
\textbf{Original (Collectivistic):}
\small \\  ``But it will not rise 
for the feet of a few. \textit{Only when all hands dig—young, old, weak, 
strong—will the spring awaken.}'' 
\vspace{3pt}
\\
\textbf{Baseline (Zero-shot LLM): }{\color{red} $\times$ Failed to transform}
\small \\ ``But it will not rise 
for the feet of a few. Only when all hands dig—young, old, weak, 
strong—will the spring awaken.''
\vspace{3pt}
\\
\textbf{Our Method: }{\color{blue} $\checkmark$ Individualistic Shift}
\small \\ ``But it will not rise 
for the feet of a few. \textit{Only when one determined soul digs will the 
spring awaken.}''
    \end{minipage}
    }
    \caption{Abduction-guided LLM transformation correctly identifies narrative element makes the target narrative shift.}
    
    \label{fig:narrative_transform}
    \vspace{-13pt}
\end{figure}

The rest of the paper is organized as follows: Section~\ref{sec:bck} provides background on individualism-collectivism as cultural dimensions and introduces our social-science grounded diagnostic survey. Section~\ref{sec:related_work} reviews related work in text style transfer, theory-grounded LLM systems, and abductive reasoning. Section~\ref{sec:logical_language} provides a logical reasoning framework for narrative shift and formalizes it as an abduction problem. Section~\ref{sec:methodology} presents our two-phase methodology—rule learning from training corpus and iterative abductive transformation. Section~\ref{sec:experimental} demonstrates the effectiveness of our approach against a zero-shot baseline across different LLMs and narrative shift directions. Finally, we conclude with our findings and future work in Section~\ref{sec:conclusion}.

\section{Background } \label{sec:bck}
Research in social sciences has shown that effective communication relies on aligning a message with an audience's narrative and worldview~\cite{GREEN20241,green2000role}. This principle extends across diverse contexts: diplomacy, journalism, and efforts to mitigate political polarization all depend on understanding how different communities interpret the same events through distinct cultural lenses. In military and intelligence contexts, information operations similarly tailor messages to specific populations. For example, state and non-state actors involved in contemporary conflicts frame the same events differently for domestic, allied, and adversary audiences based on historical memory and collective identity~\cite{paul2016russian}. Individualism and collectivism are foundational cultural dimensions that shape how the self is defined, how goals are prioritized, and how appropriate action is determined \cite{choi2025individualism,triandis1995individualism,hofstede1984culture}. Individualism is associated with an independent self-construal emphasizing autonomy, personal traits, and individual achievement, whereas collectivism is associated with an interdependent self-construal emphasizing social roles, relationships, and group belonging \cite{markus2014culture,singelis1994measurement}.

We define narrative shift to be the transformation of a text's narrative orientation while preserving the original narrative's core message and events. By refining LLMs to systematically alter narrative orientation along this cultural dimension while preserving fidelity—a capability currently absent in existing literature—we can facilitate more impactful communication across cultural divides. This allows for the intentional creation of messages that either align with or strategically contrast the narratives of specific communities. 
To operationalize these cultural dimensions for computational analysis, our team that includes social science experts developed a structured 20-item diagnostic survey grounded in cross-cultural psychology and narrative theory. The survey evaluates narratives across 20 distinct feature dimensions, 
with each feature paired bidirectionally—one question assesses individualistic framing and 
its complement assesses collectivistic framing. This results in 40 diagnostic questions (20 individualistic, 20 collectivistic). 
Individualistic questions emphasize personal goal prioritization, action guided by individual preferences, individual responsibility over outcomes, and recognition based on personal achievement \cite{mcadams1993stories,bercovitch1975puritan,moon2020small}. Collectivistic questions emphasize group goal prioritization, action guided by social norms and group expectations, collective responsibility, and recognition based on shared contribution \cite{singelis1995culture}. Each question targets a specific narrative element—such as conflict, actors, actions, goals, relationships, resolution, or moral evaluation—such that cumulative responses determine the degree to which a narrative aligns with individualistic or collectivistic orientations. 
We describe the diagnostic in more detail in Section~\ref{sec:experimental}, and include the full diagnostics in Appendix B.

\section{Related Work} \label{sec:related_work}

Recent work on zero-shot text-style transfer with LLMs~\cite{reif2022recipe} show that diverse rewriting examples in prompts can eliminate the need for complex task-specific training data. However they lack systematic guidance on what specific style elements to modify and thus suffer from hallucinations and incoherent generation. Another recent approach~\cite{suzgun2022prompt} attempt to improve style-transfer through ranking and selection of candidate transformations using contrastive prompts. However it seems to show directional biases—positive to negative transformation was easier than negative to positive. Our method replaces post-hoc ranking by abductive reasoning to identify which chunks require modification and allows for bidirectional coherent narrative shift without expensive candidate generation. An alternative line of research focuses on fine-tuning the LLMs on style-transfer data. Recent study~\cite{mukherjee2024large} comparing zero-shot, few-shot, and finetuning approaches shows that task-specific finetuning substantially improves multilingual style transfer. However, finetuning requires large amount of labeled data. Our approach does not require this task-specific finetuning and works well even with small language models. 

Recent work~\cite{singh-etal-2025-llms,mukherjee2025evaluating} show that combining multiple evaluation metrics with LLM-based judgment provide more reliable assessments of narrative quality instead of metrics or LLM judge alone. Our work extends this principle to the domain of narrative transformation. We use KL divergence to assess semantic fidelity and LLM-diagnosis to measure narrative shift. Theory-grounded LLM systems or generative agents\cite{yue2025relate,chun2025conflictlens,park2023generative} across domains have showed that constraining LLM outputs with domain specific social science principles improves coherency as well as real-world alignment. We extend this principle to narrative transformation by grounding our approach in social science theory—specifically, individualism and collectivism. 
Early work in abductive inference has been applied for movement data. It was used to generate faux trajectories that meet spatio-temporal constraints~\cite{Bavikadi_2025} and identifying future regions of an agent~\cite{Bavikadi2025SeacretAM}. On the contrary, we look at narrative shift on natural language text and use optimization to guide an LLM-based transformation.

\section{Technical Preliminaries}\label{sec:logical_language}

We consider a logical language with a sets of constants $\mathcal{C}$, variables $\mathcal{V}$, and predicates $\mathcal{P}$. Let $ \domainSet_s \cup \domainSet_c \cup \domainSet_n \cup \domainSet_f \subseteq \mathcal{C}$, where the domain $\domainSet_s$ is a set of constants denoting narrative stories, domain $\domainSet_c$ a set of constants representing unique chunks (subsets of a story) from all stories (additionally, $\domainSet_c^s \subset \domainSet_c$ denotes the set of unique chunks for the story $s \in \domainSet_s$) , domain $\domainSet_n = \{ ind, col \} $ is a set of constants indicating the individualistic and collectivistic narrative and domain $\domainSet_f = \{f_1, \ldots, f_{20}, \ldots, f_{40}\}$ is a set of constants indicating various narrative features. Correspondingly, the sets of variables are, $ \domainVar_s \cup \domainVar_c \cup \domainVar_n \cup \domainVar_f \subseteq \mathcal{V}$.

\noindent In addition to first-order logic syntax and semantics, we allow for annotations on atoms. Annotations are are elements of a lower semi-lattice structure $\mathcal{L}$, where each element is a subset interval of a unit interval $ \mu \subseteq [0,1]$, this generalizes fuzzy logic~\cite{ks92}. An atom can be represented as $a = p(t_1,\ldots,t_n)$, where $p \in \mathcal{P}$ is an $n$-ary predicate and $t_1, \ldots, t_n \in \mathcal{C} \cup \mathcal{V}$. A fact or a ground atom is an atom without any variables.
An annotated ground atom $a_\mu$ means $a$ is associated with an annotation $\mu$.  Consider $\groundatoms$ to a set of all ground atoms.
We use an annotated logic framework where the desirable annotation is positioned at top of the lattice and $[0,1]$ is the bottom element ($\bot$) of the lattice. If $\alpha,\alpha_1, \ldots, \alpha_m$ are atoms, then $\alpha \leftarrow \alpha_1 \land \ldots \land \alpha_m $ is called a rule. A program $\Pi$ is defined as set of facts and rules.
We use the notation $\Gamma^*_\Pi$ to refer to the deductive closure with the logic program $\Pi$ and initial set of facts with assigned annotations to $\bot$ (meaning that we start deduction under the assumption of uncertainty). We perform deductive inference with logic program $\Pi$ \cite{ssTAI22} and is efficiently implemented using PyReason \cite{aditya2023pyreason}.

\begin{example}[Language]
\label{ex:lang}
A narrative story $s \in \domainSet_s$ can be expressed as a series of chunks $\langle c_1,c_2,.. \rangle$ that have certain narrative characteristics, where $\{ c_1,c_2 \ldots \}\in \domainSet_c^s $. We represent the chunks as $\{ \contains(s, c_1)_1, \contains(s, c_2)_1 \ldots \}$ and $Ch_s$ is a set of chunks from $\domainSet_c$ in story $s$. Suppose $s$ is a collectivistic narrative story. We assume binary predicates $\{ \sfeat, \cfeat, \head, \contains,\\ \assoc \} \in \mathcal{P}$.  For our use-case, we assign the annotation with a scalar- the lower bound of the annotation intervals, as they always have the upper bound as 1 or is the same as the lower bound.
     Each natural language text or chunk can be related with multiple narrative features denoted by constants $\domainSet_f$.  We use a trainable generative function $g_\theta(x)= \mu; x \in \domainSet_s \cup \domainSet_c$ to obtain a rating $\mu \in \{ 0.2,0.4,0.6,0.8,1 \}$ on features for a given text. We notate this in our language using ground atoms $\cfeat(s, ind)_{0.6}$ for $s \in \domainSet_s$, $ind \in \domainSet_n$ showing that the story is of a individualistic narrative with a confidence of $0.6$, and $\cfeat(c, f_{20})_{0.2}$ for $c \in \domainSet_c$, $f_{20} \in \domainSet_f$ showing that the story has narrative characteristic $f$, with a confidence of $0.2$. Also, we use $\assoc(ind, f_{20})_1$ to show that the feature $f_{20}$ is associated with the individualistic narrative. We have $\{ f_1, \ldots, f_{20} \}$ associated with the individualistic narrative and $\{ f_{20}, \ldots, f_{40} \}$ associated with the collectivistic narrative. Finally, to represent the confidence of the story having an individualistic narrative, we use ground atom, $\head(x, f_{20})_\conf$ for $x \in \domainSet_c \bigcup \domainSet_s,$ and $\conf \in [0,1]$.

\end{example}
\paragraph{Abduction.} 
 Given a story, our goal is to transform it's narrative by identifying a set of chunks to modify along with the corresponding characteristics.
We formalize our problem as an abduction problem: Given a set of observations $\mathcal{O}$ that are annotated facts depicting the various chunks and corresponding narrative characteristics of the story, hypothesis $\mathcal{H}$ consisting of annotated facts that have the same chunks in $\mathcal{O}$ with different narrative characteristics, a logic program $\Pi $ governing the narrative structure of the story. This gives us an instance of an abduction problem $\langle \mathcal{O},\mathcal{H},\Pi \rangle$. We abduce an explanation $\mathcal{E} \subseteq \mathcal{H}$ to the abduction problem $\langle \mathcal{O},\mathcal{H},\Pi \rangle$, such that $\Pi  \bigcup \mathcal{O} \bigcup \mathcal{E}$ is consistent and $\mathcal{E}$ maximizes the parsimony function $\parsimony$. 
We now formally define Observations $\observationSet$, Hypothesis $\hypothesisSet$, and Logic Program $\Pi$.

\begin{definition}[Observations ($\observationSet$)]
    Given $s \in \domainSet_s$, a narrative diagnosis function, $diagnosis: \domainSet_s \rightarrow \groundatoms $, generates set of ground atoms, then $\mathcal{O}$, is defined as:
\begin{align*}
\observationSet = \forall_{c \in \domainSet_c^s}\{ \contains(s, c) \} \cup diagnosis(s) 
\end{align*}
\begin{align*}
diagnosis(s) = \forall_{c \in \domainSet_c^s}
&\{ \cfeat(c, f)_\mu 
\\
&\mid  f \in \domainSet_f, \mu \in \{0, 0.2, 0.4, 0.6, 0.8, 1\}\} 
\end{align*}
\end{definition}

We further elaborate on the narrative $diagnosis$ function in Section~\ref{sec:experimental}.

\begin{definition}[Hypothesis ($\mathcal{H}$)]
    For a story $s \in \domainSet_s, \mathcal{H}$ is the set of all possible atoms (that potentially contribute to the narrative transformation) indicating different chunks with various narrative characteristics. We define $\mathcal{H}$ as follows:
\begin{multline*}
\mathcal{H} = \forall_{f \in \domainSet_f} \{\cfeat(c, f)_{\mu'} \mid \\
\mu' \geq \max(r,0), \cfeat(c, f)_\mu \in \mathcal{O}\}
\end{multline*}
\end{definition}

\begin{definition}[Logic Program ($\Pi$)]
    For $S \in \domainVar_s, C \in \domainVar_c, F \in \domainVar_n, F'  \in \domainVar_f$, $\mu \in \{0,0.2,0.4,0.6,0.8,1\}$, and $c \in [0, 1]$, we define 2 types of non-ground rule templates:

\begin{equation*}\label{eq:rule_template}
\textit{R1 : }\head(S, F)_{c} \leftarrow \sfeat(S, F)_{\mu'}
\end{equation*}

\begin{multline*}\label{eq:rule_template2}
\textit{R2 : }\sfeat(S, F)_{\mu'} \leftarrow \cfeat(C, F')_{\mu} 
\\ \land 
\contains(S,C)_{1} \land \assoc(F,F')_1
\end{multline*}
\end{definition}

Consider an explanation $ \explanationSet = p(c,f):\mu \in \Pi $ formed with a binary predicate $p = \cfeat$, and a corresponding set of constants $\mathcal{S} = \domainSet_{f} \setminus \{ f \}$ (which we call suggestive features). We define a parsimony function, $\parsimony$, that outputs a scalar value (which we call the corpus similarity score)  indicating the cultural relevance for a given story $s \in \domainSet_s$.

\begin{align*}
\parsimony_{\langle \Pi, s, f \rangle}(\explanationSet) = 
& \Gamma^*_{\Pi \bigcup \explanationSet \bigcup \observationSet}
  (\head(s, f)) \\
& - \Gamma^*_{\Pi \bigcup \observationSet}
  (\head(s, f))
\end{align*}

We define $\domainSet_{c,\mathcal{O}}$ to indicate the set of constants from the ground atoms of $\observationSet$ that belong to $\domainSet_c$. Similarly, we also define $\domainSet_{f,\mathcal{O}}, \domainSet_{f,\mathcal{E}}$.
 We define a function $\phi_{k}$ as the optimization function to solve our abduction problem that forms a set of new facts of the form $p(c_k,f_i):\mu_i$ where $f_i \in \mathcal{S}, c_k \in \domainSet_{c,\mathcal{O}}$, that maximizes the corpus similarity score.

\begin{equation*}
   \begin{aligned}
    \phi_k(c, p) = \argmax_{\substack{\mathcal{E} \subset \mathcal{H} \setminus \mathcal{O} \\
|\mathcal{E}| = k, \, f \in \domainSet_{f,\explanationSet}}}
\parsimony_{\langle \Pi \cup \mathcal{E}, c, f \rangle}(\explanationSet) 
\\ \quad \textit{s.t. } \quad
 \Pi \cup \explanationSet \cup \observationSet \text{ is consistent}
\end{aligned} 
\end{equation*}

We note that we formalize the rule templates and $\hypothesisSet$ to guarantee the consistency of $\Pi \cup \observationSet \cup \explanationSet$. As $\forall\mu ( a_{\mu} \in \observationSet \implies \exists \mu'>\mu \mid a_{\mu'} \in \hypothesisSet )$, the explanation will always have an annotation that is above the annotation of the corresponding ground atom in the observation. $\Pi \cup \observationSet \cup \explanationSet$ is not inconsistent as $\Gamma_{\Pi \cup \observationSet \cup \explanationSet}$ does not result in deducing for an atom assigned with a new annotation that is lower than the initial annotation. Using the lower lattice structure $\lattice$, for the annotations and the rule templates, result in annotations of $\Gamma_{\Pi}$ to be higher in the lattice that the initial set of facts.

\begin{example}[Abduction]
\label{ex:abd}
    Building upon the Example~\ref{ex:lang}, $\mathcal{O}$ is a set of annotated facts with each fact representing a chunk in the given story $s$, $\{\cfeat(c_1, \texttt{individual\_accolades})_0, \\ \cfeat(c_1, \texttt{relationship\_framing})_{0.4}, \\ \contains(s, c_1)_1,\cfeat(c_2, \texttt{uniqueness})_{0.2}, \\\contains(s, c_2)_1 \ldots \}$ Hypothesis is also a set of annotated facts, $\mathcal{H} = \{ \cfeat(c_1, \texttt{uniqueness})_1 \} $, representing hypothetical features of the chunks to change the overall stories narrative to be individualistic. $\Pi$ is the set of annotated logical rules that specify the individualistic narrative of a given story like: 
    \begin{align*}
         \sfeat(s, ind)_{0.8} \leftarrow \cfeat(c_1, \texttt{uniqueness})_{0.8} \\ \land \contains(c_1,s)_1 \land \assoc(ind, \texttt{uniqueness})_1
    \end{align*} 
    Here, $c_1$ of the $s$ frames the individual as self-reliant- a quality associated with an individualistic narrative, then the story is of individualistic narrative with a confidence of $0.8$. That leads to $s$ being similar to the training corpus with the confidence of $0.56$ indicated by the rule:
    \begin{align*}
        \head(s, ind)_{0.56} \leftarrow \sfeat(s, ind)_{0.8}
    \end{align*}
    Explanation $\mathcal{E} = \{ \cfeat(c_1, \texttt{uniqueness})_{0.8} \} $ is a suggested set of chunks that to be modified to change the narrative characteristics of the story.
\end{example}

\section{Methodology}\label{sec:methodology}
We consider two narrative transformation tasks both formalized as abduction problem $\langle \observationSet, \hypothesisSet, \logicProgram \rangle$. Given a story $s \in \domainSet_s$, we search for an explanation $\explanationSet \subseteq \hypothesisSet$ such that when chunks from $\explanationSet$ are transformed, the story shifts toward a target narrative orientation. 
Our system combines symbolic reasoning with large language models through a two-phase architecture. Phase 1 uses training corpus to learn rules which are then used in Phase 2 to perform abductive-guided transformation of a test story as seen in Figure~\ref{fig:system_diagram}.

\begin{figure}[h]
    \centering
    \includegraphics[width=1\linewidth]{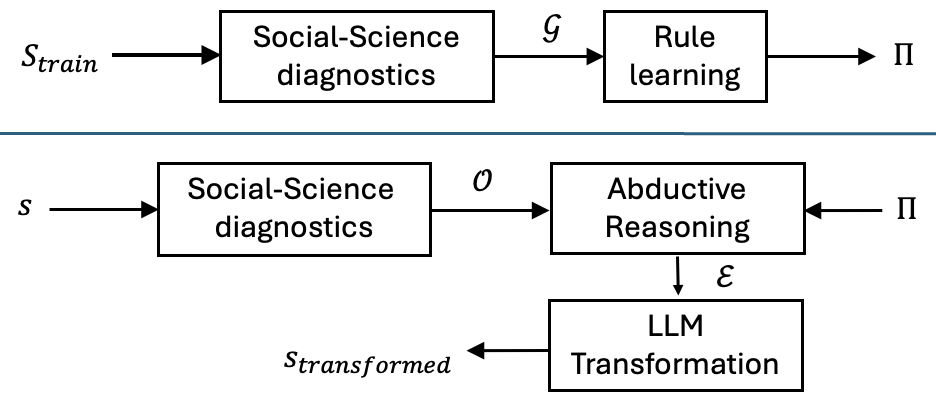}
    \caption{Abduction-guided Transformation (Our Approach).}
    \label{fig:system_diagram}
\end{figure}

\paragraph{Phase 1: Rule learning.}

We have a training corpus $S_{\text{train}} \subseteq \domainSet_s$ of stories from a target narrative orientation. We get the observations $\observationSet$ from all the training stories as follows.
\begin{align*}
\observationSet_{train} = \bigcup_{s \in S_{train}} diagnosis(s)  \cup \forall_{c \in \domainSet_c^s}\{ \contains(s, c) \} 
\end{align*}
We then infer the set of ground atoms with a predicate $q = \sfeat $ by one application of fixpoint operator, i.e. 
\[
\Pi_{q} = \fixpointOperator^*_{\observationSet_{train} \cup R1}(q(s,f))
\]
Now, for each feature $f \in \domainSet_f$ and each rating value $\mu \in \{0, 0.2, 0.4, 0.6, 0.8, 1\}$, 
we compute the confidence from observations:

\[
\textit{conf}(f, r) = \frac{|\{s \in S_{\text{train}} : \exists q(s, f)_\mu \in \Pi_{q}|}{|S_{\text{train}}|}
\]
For each $(f, \mu)$ pair, we extract a rule of the form $R2$:

\[
\texttt{corpus\_sim}(X, f)_{conf(f, \mu)} \leftarrow q(X, f)_{\mu}
\]

\paragraph{Phase 2: Iterative Abductive Transformation.}
For a test story $s_0 \in \domainSet_s$ requiring narrative transformation, we perform iterative abductive reasoning as described in Algorithm~\ref{alg:iterative_abduction}. At each iteration $t$, we solve the abduction problem $\langle \mathcal{O}_t, \mathcal{H}, \Pi \rangle$ 
as formalized in Section~\ref{sec:logical_language} to identify the chunks to transform.   
We define a narrative transformation function that modifies narrative content:
\[
\textit{llm\_transform} : \domainSet_c \times \domainSet_f \times \domainSet_s \times [0,1] \rightarrow \domainSet_s
\]
The transformation function takes as input a chunk $c \in \domainSet_c$, a feature $f \in \domainSet_f$, 
a story $s \in \domainSet_s$, and a target annotation value $\tau \in [0,1]$. It outputs a modified 
story $s^{\text{transformed}} \in \domainSet_s$ where the selected chunk has been modified by an LLM 
to shift the narrative toward the target annotation. The LLM prompt is fixed for a given transformation 
direction but may differ for different directions. The LLM prompt to transform particular identified story chunk given the current narrative and target narrative is provided in the Appendix C. 

\begin{algorithm}
\caption{Iterative Abductive Transformation}\label{alg:iterative_abduction}
    \begin{algorithmic}[1]
        \Require{$s_0$, $\logicProgram$, $T_{\max}$, $k$}
        \Ensure{Optimal transformed story $s_{t^*}$ and optimal explanation $\explanationSet_{t^*}$}

        \State $t \gets 0$
        \While{$t < T_{\max}$}
            \State $\observationSet_t \gets \textit{diagnosis}(s_t)$

            \State $\explanationSet_t \gets \phi_k$
            \Comment{Solve $\langle \observationSet_t, \mathcal{H}, \logicProgram_{\text{train}} \rangle$ as defined in Section \ref{sec:logical_language}}

            \State $C_t \gets \{c \in \domainSet_c \mid \exists (f, \mu) \texttt{c\_feature}(c, f)_{\mu} \in \explanationSet_t\}$

            \For{ $c \in C_t$}
            \State $f, \tau \leftarrow \textit{extractFeature}(\explanationSet_t)$
            \State $s_t \gets \textit{llm\_transform}(c, f, s_t, \tau)$
            \EndFor

            \State $t \gets t + 1$

        \EndWhile

        \State $t^* \gets \arg\max_t \{\text{median}(\observationSet_t)\}$
        \State \Return{$s_{t^*}$, $\explanationSet_{t^*}$}
        \end{algorithmic}
\end{algorithm}

 \paragraph{Baseline. }
We implement a baseline transformation function using a zero-shot LLM approach. 
 The baseline prompt is: \textit{``Make the following story more  [individualistic/collectivistic]''} where the direction is determined by the target transformation task.
\paragraph{Cost Analysis}
We show theoretically the bound on the number of LLM calls with Proposition~\ref{prop:cost} (proof in the Appendix A) and that it increases linearly with the number of abduced story chunks.
\begin{proposition}\label{prop:cost}
    Let $s$ be a story represented as a sequence of tokens and $\chunkStory$ denote 
the set of chunks from $s$. For each chunk $c \in \chunkStory$, 
let $\text{size}(c)$ denote the number of tokens in $c$.  
Then total number of LLM transformations is bounded by following quantity.
\[
N_c = \sum_{c \in \chunkStory} \text{size}(c)
\]


\end{proposition}


\section{Experimental Results}
\label{sec:experimental}

In this section, we present the experiments conducted to validate our approach. We establish our main evaluation metrics and discuss our findings. We perform a hyperparameter sensitivity study before concluding with our experimental results. 

\paragraph{Setup. }
We create a dataset of 118 natural language narratives (involving 90 individualistic and 28 collectivistic narratives). 
The individualistic narrative stories in our study were selected based particularly on narratives following a ``rags to riches'' trajectory common in Western literary traditions \cite{mcadams1993stories} and sourced from classic American literature as well as contemporary biographical narratives of successful figures. Collectivistic narratives, which are less prevalent in Western written fiction, were sourced from diverse geographical and cultural contexts including African, Asian, and Latin American traditions, as well as Western narratives that foreground social relationships and critique individualism \cite{achebe2012hopes,ba1981living,klinkowitz2004vonnegut}.
\paragraph{Narrative Diagnosis. } 
We construct a social science theory based narrative diagnosis survey elaborated in Section~\ref{sec:bck}. The answers to the survey are in the scale of 1-5 with higher scores indicating the higher individualistic or collectivistic perspective of the narrative. Given a candidate story, we use an LLM to answer the survey to obtain ratings of the same story 10 times.  We repeat this for some stories and found stable results on the ratings as seen in Figure~\ref{fig:llm_stability}. Note that these ratings are normalized to form the annotations in the output of $diagnosis(s)$ introduced in Section~\ref{sec:logical_language}.
\begin{figure}[h!]
    \centering
    \includegraphics[width=0.9\linewidth]{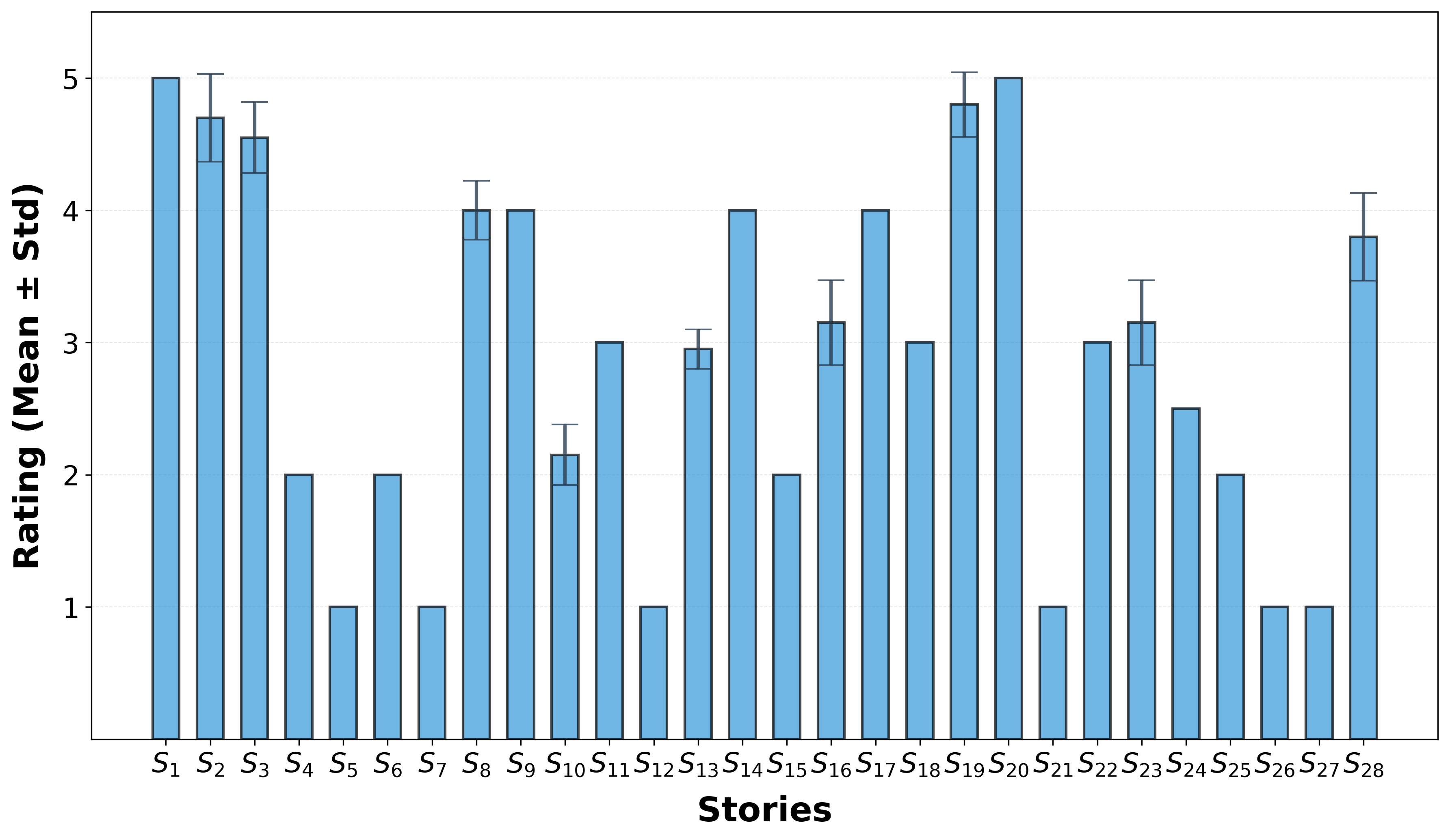}
    \caption{Stability of Social-Science LLM diagnostics over 10 runs for each story.}
    \label{fig:llm_stability}
\end{figure}

\paragraph{Evaluation Metrics. }We use two metrics. First, to capture the narrative of a story we define diagnosis score-- given a narrative, we use a diagnosis to evaluate how individualistic or collectivistic the narrative is based on a survey.  We compute the median of the 20 diagnosis scores from the narrative diagnosis survey (explained in the Section~\ref{sec:bck}) on transformed stories and report the average improvement over the original in the Table~\ref{tab:forward_results} and ~\ref{tab:inverse_results} for different LLMs. We also report the mean, median and mode ratings to compare zero-shot approach and our approach for all the stories (Figures in Appendix D).

The second metric we use is the Kullback-Leibler (KL) divergence to analyze the content similarity between the transformed story and the original story. Let  $S^{trans}$ and $S$ be probability distributions of transformed stories and the original story over a shared vocabulary $\mathcal{X}$, where we employ additive smoothing with $\alpha = 10^{-5}$.
Here, $S^{trans}(x)>0 \implies S(x)>0$ for all $x\in \mathcal{X}$. While  We define KL divergence as:
\begin{equation*}
     D_{\mathrm{KL}}(S^{trans}(x) \,\|\, S(x)) = \sum_{x \in \mathcal{X}} S_{trans}(x)\log\left(\frac{S^{trans}(x)}{S(x)}\right) 
\end{equation*}
We use KL divergence to depict the deviation of the transformed stories through baseline, as well as our approach from the original stories across different language models.

\paragraph{Implementation Details. }For all our experiments, we used Apple M4 Pro with 20-core architecture and 24GB RAM. We use reasoning models GPT-4o, Deepseek-R1, Grok-4 fast reasoning model and a non-reasoning model Llama-4 Maverick 17B 128E Instruct v1.
The task is to transform a story with a collectivistic to individualistic narrative (C$\rightarrow$I)  unless specified the other way around (I$\rightarrow$C). We use following methods:

\textit{Zero-shot Transformation (Baseline).} We use a single task instance to transform the story to a different narrative as a prompt to the LLM.

\textit{Abduction-guided Transformation. }Our algorithm provides suggestion to change characteristics of the story for the transformation. We do this for 3 iterations and report the results of the best of the iterations.

\paragraph{Results of Diagnosis Score and Content Similarity. }
\begin{figure}[h!]
    \centering
    
    
    
    \fbox{
    \begin{minipage}{0.45\textwidth}
        \small
        \textbf{Original (Individualistic):}
        \\
        \small ``She struck another match against the wall. It burned brightly, and when the light hit the wall, it turned see-through like a thin curtain, and she could see inside a room. A snow-white tablecloth was spread out on the table, and on top sat a fancy dinner set. A roast goose steamed beautifully, stuffed with apples and prunes.''
        \vspace{2pt}
        \\
        \textbf{Baseline (Zero-shot LLM): }{\color{red} $\times$ Maintains individualism}
        \\
        \small ``She struck another match, and this time, the wall became translucent, revealing a room filled with warmth and comfort. A table was set with a snow-white cloth, and a roast goose steamed invitingly. In her vision, the goose approached her, but when the match burned out, the image faded, leaving only the cold wall.'' 
        \\
        \textbf{Our Method: }{\color{blue} $\checkmark$ Collectivistic Shift}
        \\
        \small ``She struck another match against the wall. It burned brightly, and when the light hit the wall, it turned see-through like a thin curtain, and she could see inside a room \textit{where an entire community was gathered, intertwined in a tapestry of shared stories and warmth, their faces aglow with the deep joy that comes from collective unity and mutual support.}''
    \end{minipage}
    }
    
    \vspace{0.3cm}
    
    \fbox{
    \begin{minipage}{0.45\textwidth}
        \small
        \textbf{Original (Individualistic):}
        \\
        \small ``But in the corner, leaning against the wall, sat the little girl—with red cheeks and a smiling face—frozen to death on the last evening of the old year. The New Year's sun rose on the small, heartbreaking figure.''
        \vspace{2pt}
        \\
        \textbf{Baseline (Zero-shot LLM): }{\color{red} $\times$ Preserves individualism}
        \\
        \small ``As the community awoke to the New Year, they found the little girl in the corner, her face serene and smiling, though she had succumbed to the cold. The people gathered around, moved by the sight.''
        \vspace{2pt}
        \\
        \textbf{Our Method: }{\color{blue} $\checkmark$ Collectivistic Shift}
        \\
        \small ``But in the corner, leaning against the wall, \textit{the community had gathered around the little girl—forming a circle of warmth and unity, ensuring she was embraced by their collective love and care on the last evening of the old year.}''
    \end{minipage}
}
    
    \caption{Abduction-guided LLM transformation successfully shifts narrative from I$\rightarrow$C across multiple story segments.}
    \label{fig:i_to_c_transform}
    \vspace{-10pt}
\end{figure}

\begin{table}[t]
\centering
\caption{Average improvement over original stories for C$\rightarrow$I transformation. Best performance per LLM is underlined; overall best is bolded.}
\label{tab:forward_results}
\begin{tabular}{lll}
\toprule
\textbf{Approach} & \textbf{Base Model} & \textbf{ Improvement (\%)} \\
\midrule
Zero-Shot       & GPT-4o        & 26.73\\
Abduction (Ours)       & GPT-4o        & \underline{\textbf{97.12}} \\
\midrule
Zero-Shot       & Grok-4          & 61.50 \\
Abduction (Ours)       & Grok-4          & \underline{94.33} \\
\midrule
Zero-Shot       & Llama-4       & 45.76 \\
Abduction (Ours)       &  Llama-4       & \underline{46.54} \\
\midrule
Zero-Shot       & DeepSeek-R1   & 56.79 \\
Abduction (Ours)       & DeepSeek-R1  & \underline{71.91} \\
\bottomrule
\end{tabular}

\end{table}

\begin{figure}[ht]
\centering
\vspace{-1em}
\begin{tabular}{@{}c@{\hspace{0em}}c@{}}
\begin{tabular}{@{}c@{}}
    \textbf{GPT-4o} \\[0em]
        \includegraphics[width=0.47\linewidth]{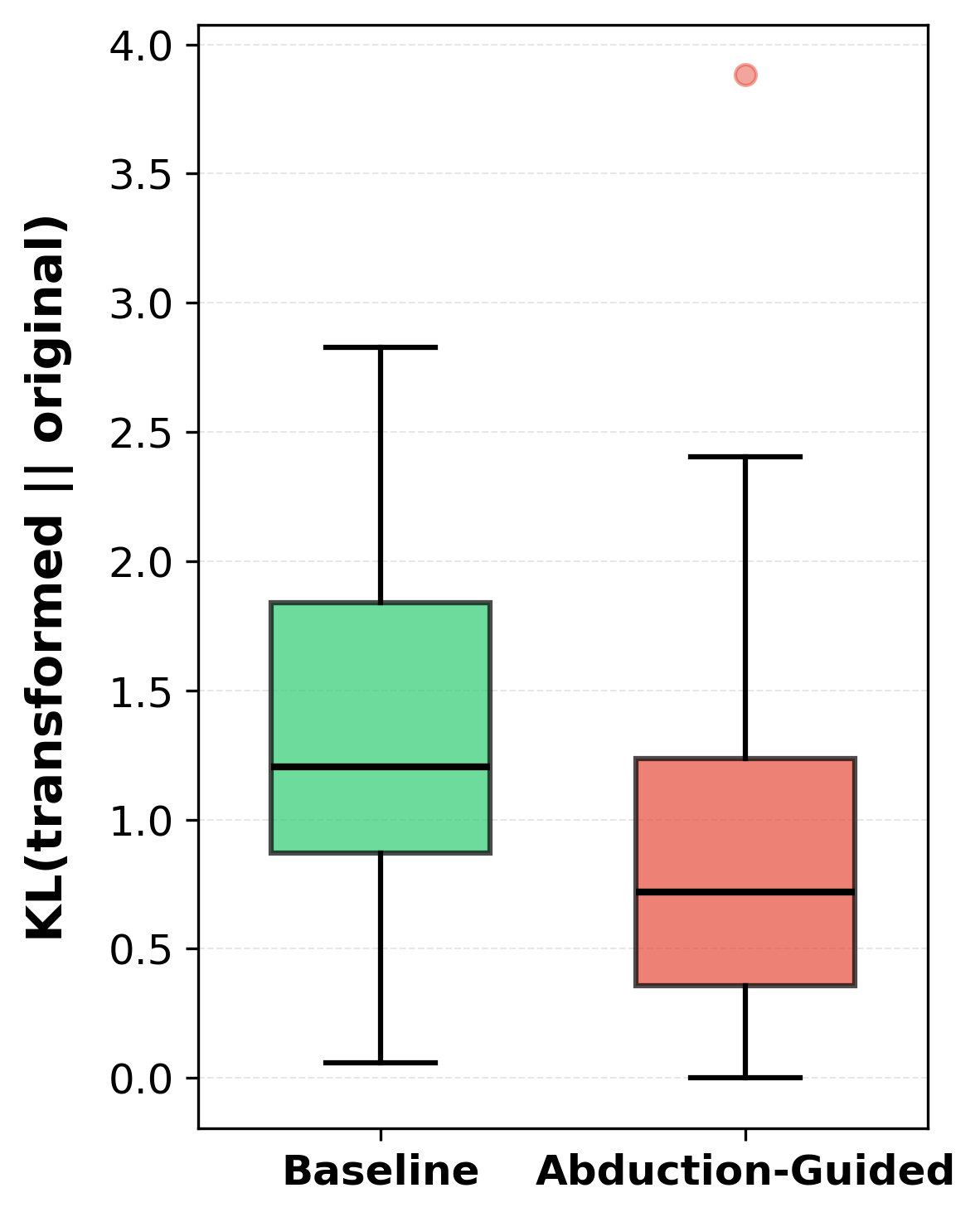}\\[0.0em]
\end{tabular}
&
\begin{tabular}{@{}c@{}}
    \textbf{Grok-4} \\[0em]
        \includegraphics[width=0.47\linewidth]{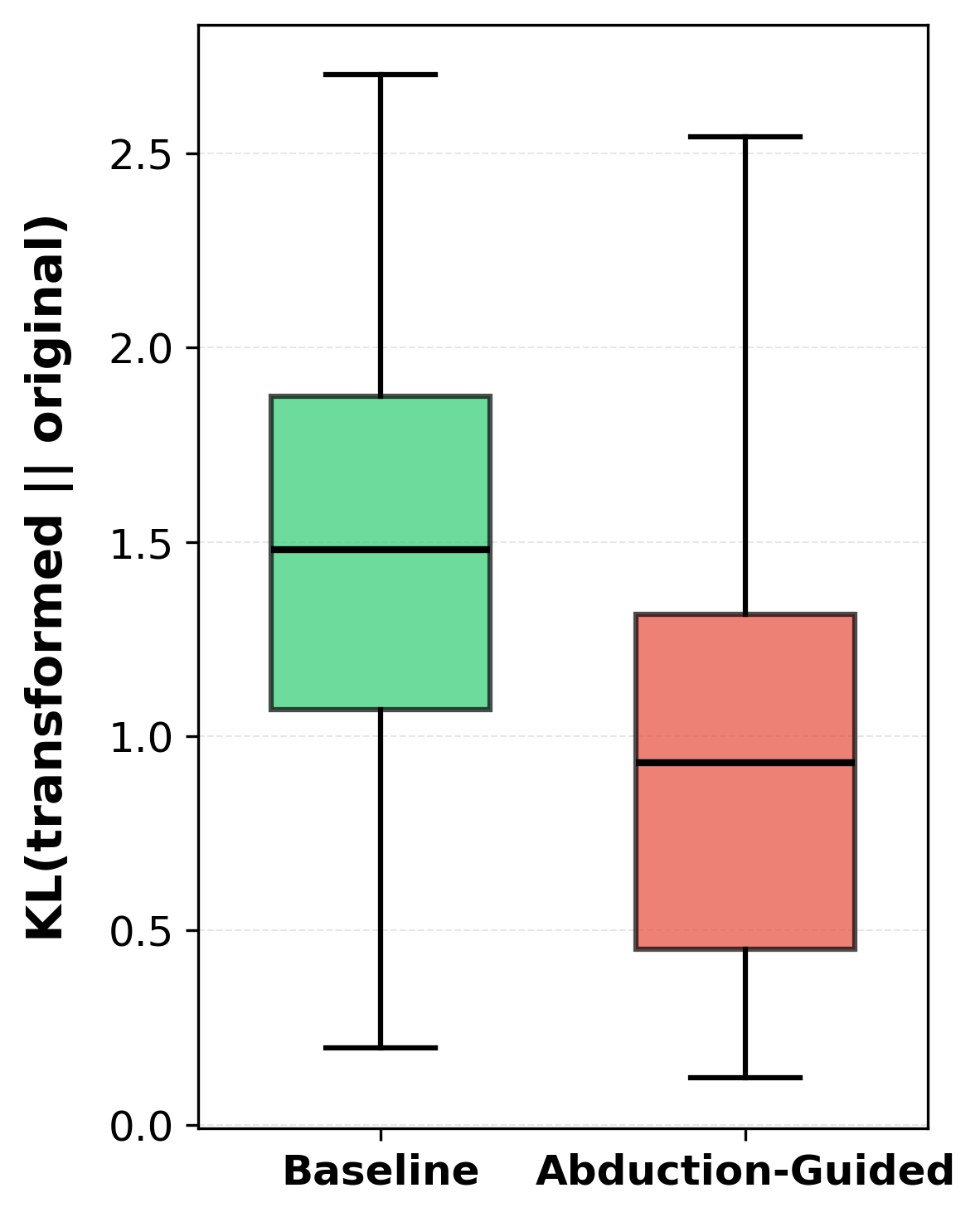}
\end{tabular}
\\[0em]
\begin{tabular}{@{}c@{}}
    \textbf{Deepseek-R1} \\[0em]
        \includegraphics[width=0.47\linewidth]{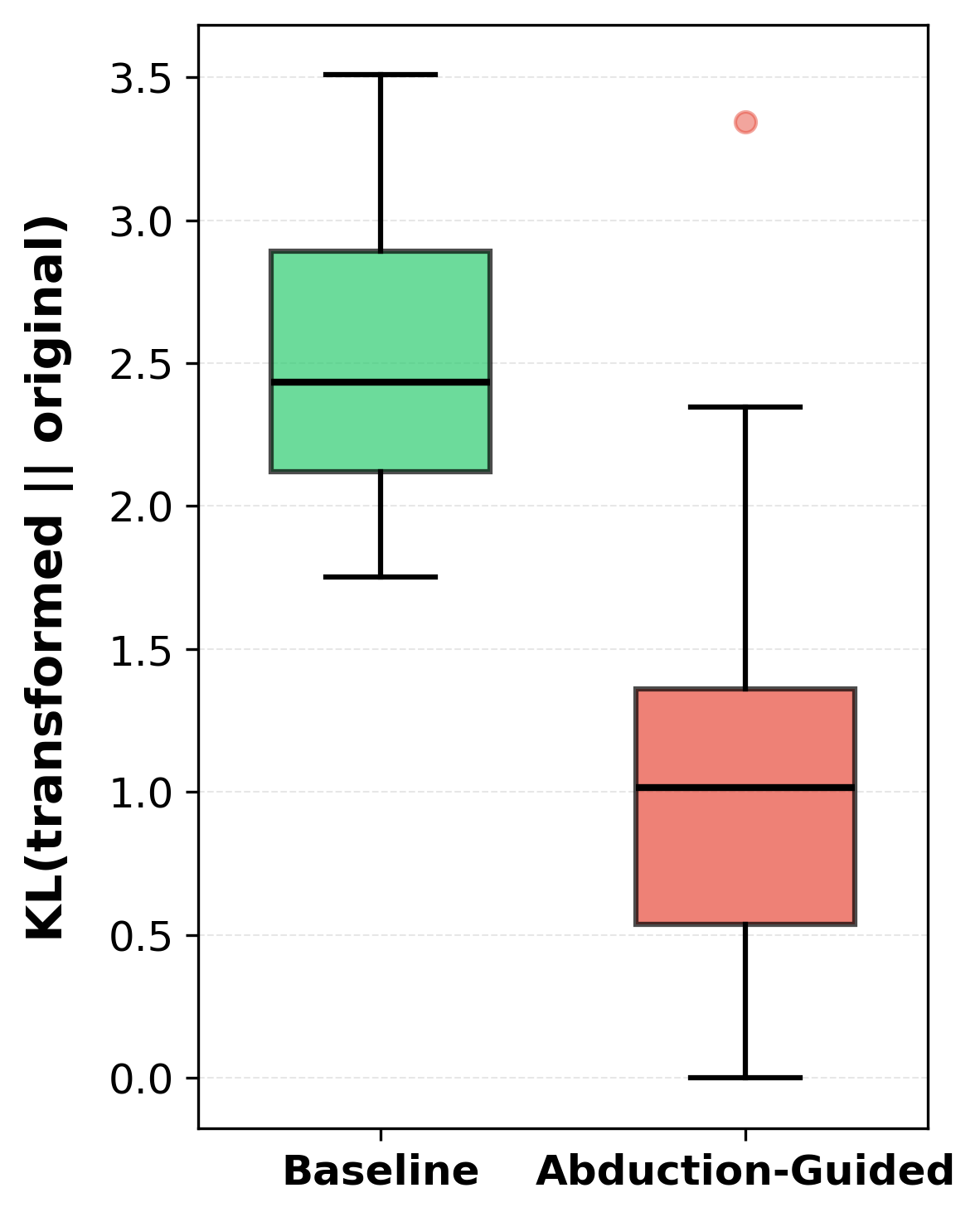}\\[0.0em]
\end{tabular}
&
\begin{tabular}{@{}c@{}}
    \textbf{Llama-4} \\[0em]
        \includegraphics[width=0.47\linewidth]{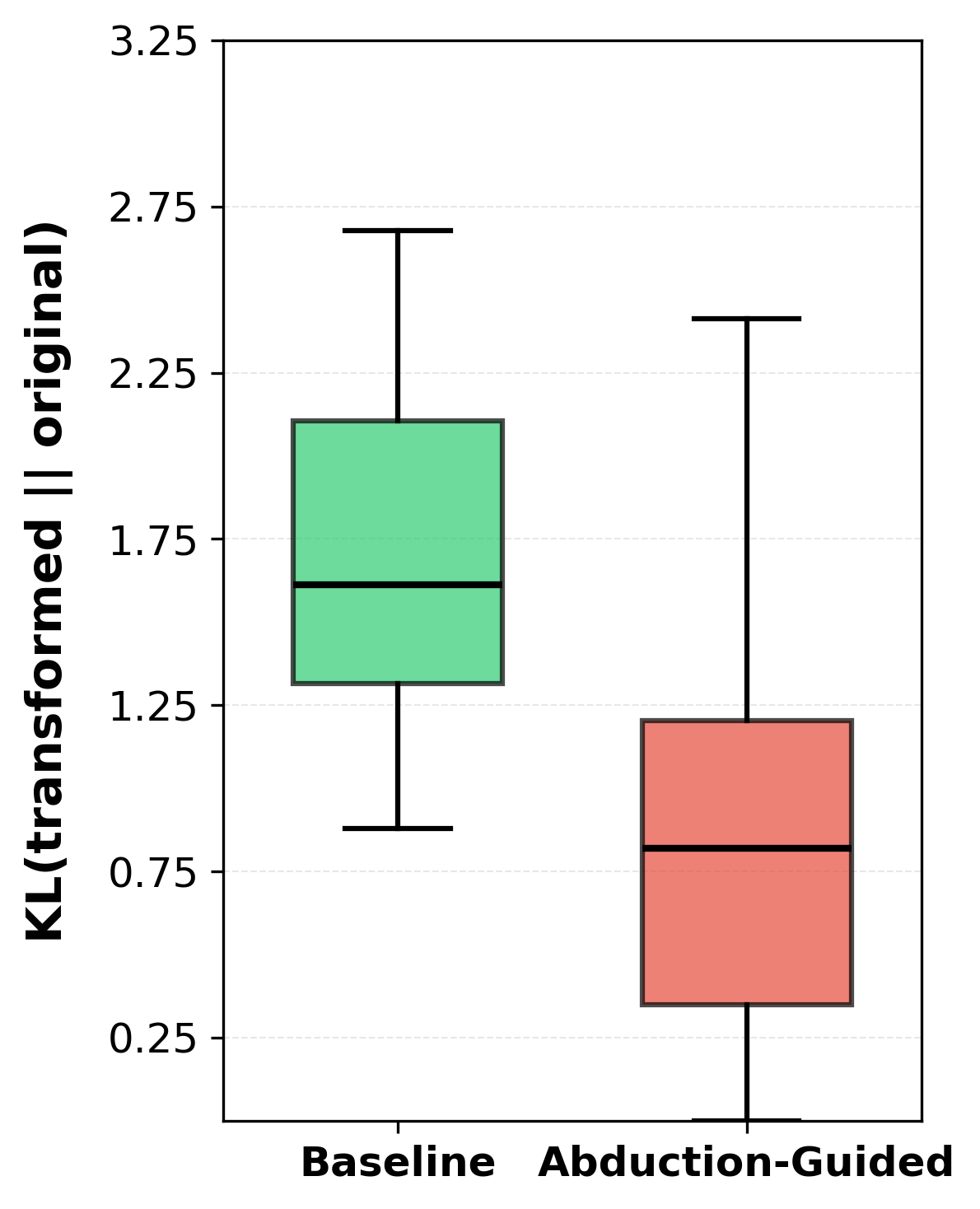}
\end{tabular}
\end{tabular}
\caption{The KL divergence across the stories for C$\rightarrow$I narrative shift. Top left: GPT-4o, Top right: Grok-4, Bottom left: Deepseek-R1, Bottom right: Llama-4.}
\label{fig:kl_comparison}
\end{figure}

We analyze the diagnosis scores of the transformed stories from all approaches in comparison to the original stories. Further, we relate these findings to the content similarity with the original stories of the transformations.

For C$\rightarrow$I narrative shift, the transformed stories from the abduction-guided method seems to give a higher individualistic rating with an improvement (over the original) of $75.44\%$ across all the models. 
Further, GPT has the highest improvement over original by $97.12\%$ as well as baseline by $55.88\%$, while maintaining an improvement in content similarity with the original stories by $40.40\%$ than the baseline (as seen in Figure~\ref{fig:kl_comparison}). 
Comparable results were observed for Grok and Deepseek. 
Although Llama only had an improvement of $7.63\%$ over the baseline, it preserved the similarity with original story by $49.16\%$. 
Moreover, the baseline consistently gave a higher divergence from the original stories while our approach preserves more information from the original story.

\begin{table}[t!]
\centering
\caption{Average improvement over original stories for I$\rightarrow$C transformation. Best performance per LLM is underlined; overall best is bolded.}
\label{tab:inverse_results}
\begin{tabular}{lll}
\toprule
\textbf{Approach} & \textbf{Base Model} & \textbf{ Improvement (\%)} \\
\midrule
Zero-Shot       & GPT-4o        & 75.00\\
    Abduction (Ours)       & GPT-4o   & \underline{95.39} \\
\midrule
Zero-Shot       & Grok-4          & 86.63 \\
Abduction (Ours)       &   Grok-4         & \underline{\textbf{97.32}} \\
\midrule
Zero-Shot       & Llama-4       & 71.43 \\
Abduction (Ours)       & Llama-4      & \underline{85.71} \\
\midrule
Zero-Shot       & DeepSeek-R1   & \underline{87.67} \\
Abduction (Ours)       & DeepSeek-R1  & 73.68 \\
\bottomrule
\end{tabular}

\end{table}

\begin{figure}[ht!]
\centering
\vspace{-1em}
\begin{tabular}{@{}c@{\hspace{0em}}c@{}}
\begin{tabular}{@{}c@{}}
    \textbf{GPT-4o} \\[0em]
        \includegraphics[width=0.47\linewidth]{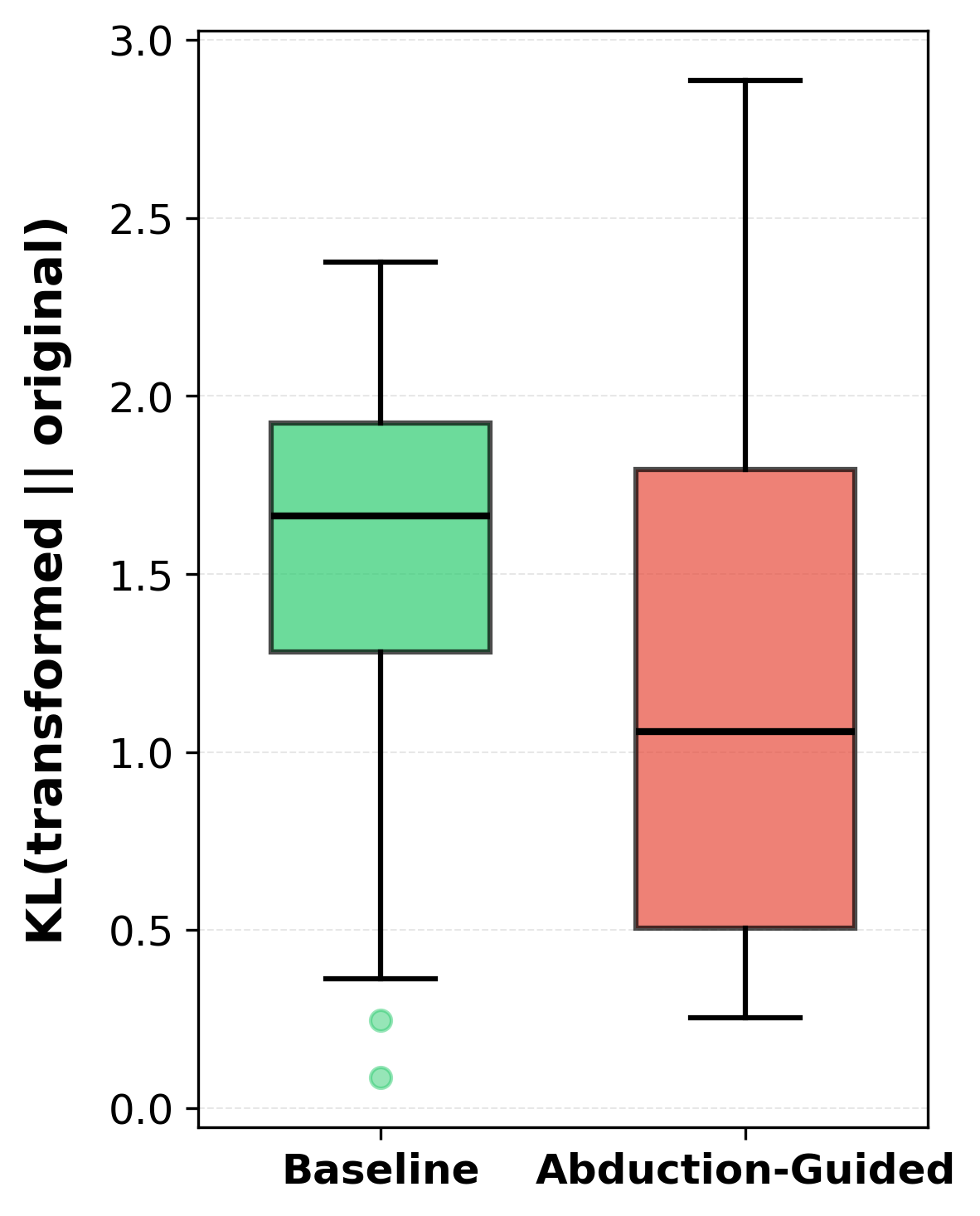}\\[0.0em]
\end{tabular}
&
\begin{tabular}{@{}c@{}}
    \textbf{Grok-4} \\[0em]
        \includegraphics[width=0.47\linewidth]{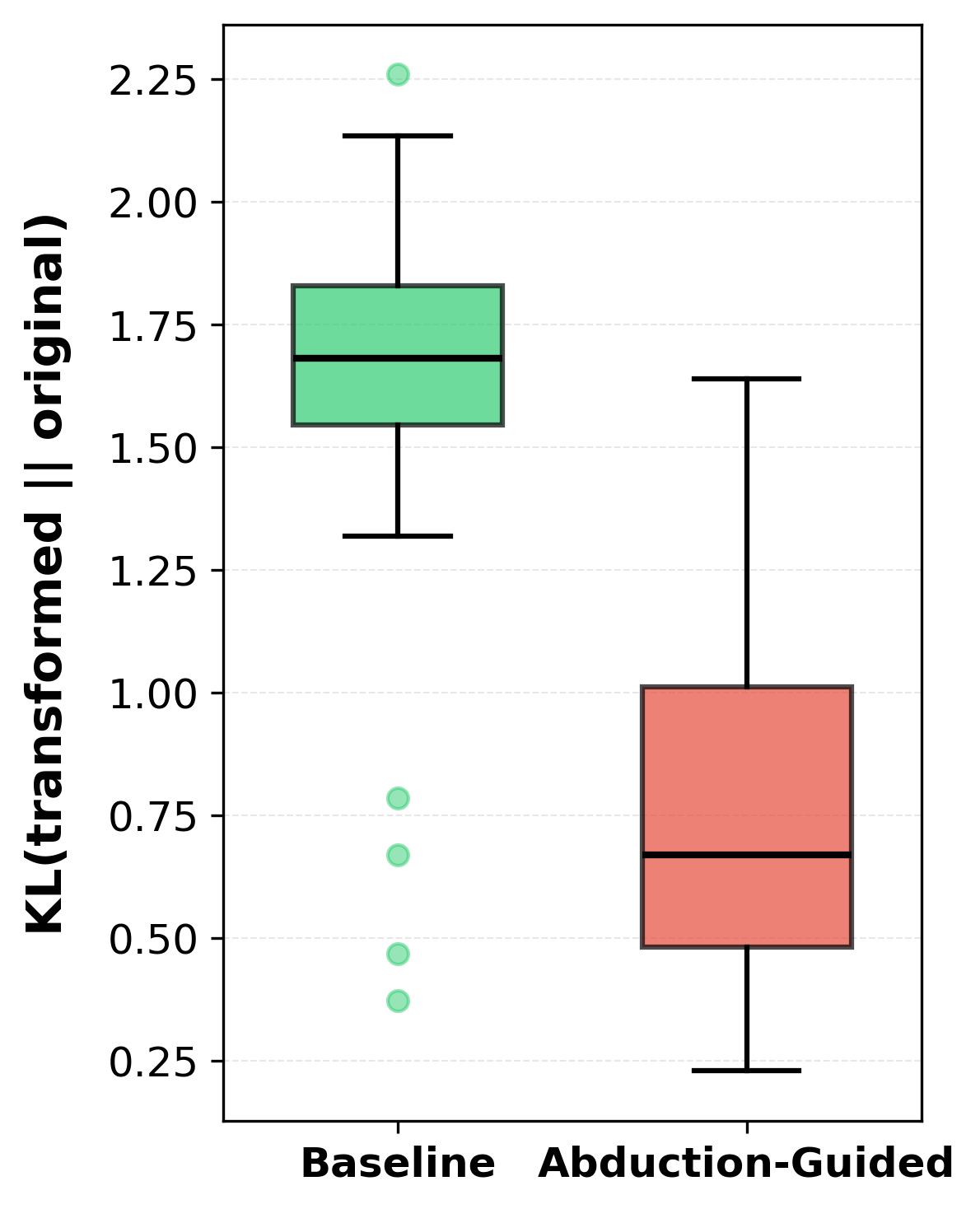}
\end{tabular}
\\[0em]
\begin{tabular}{@{}c@{}}
    \textbf{Deepseek-R1} \\[0em]
        \includegraphics[width=0.47\linewidth]{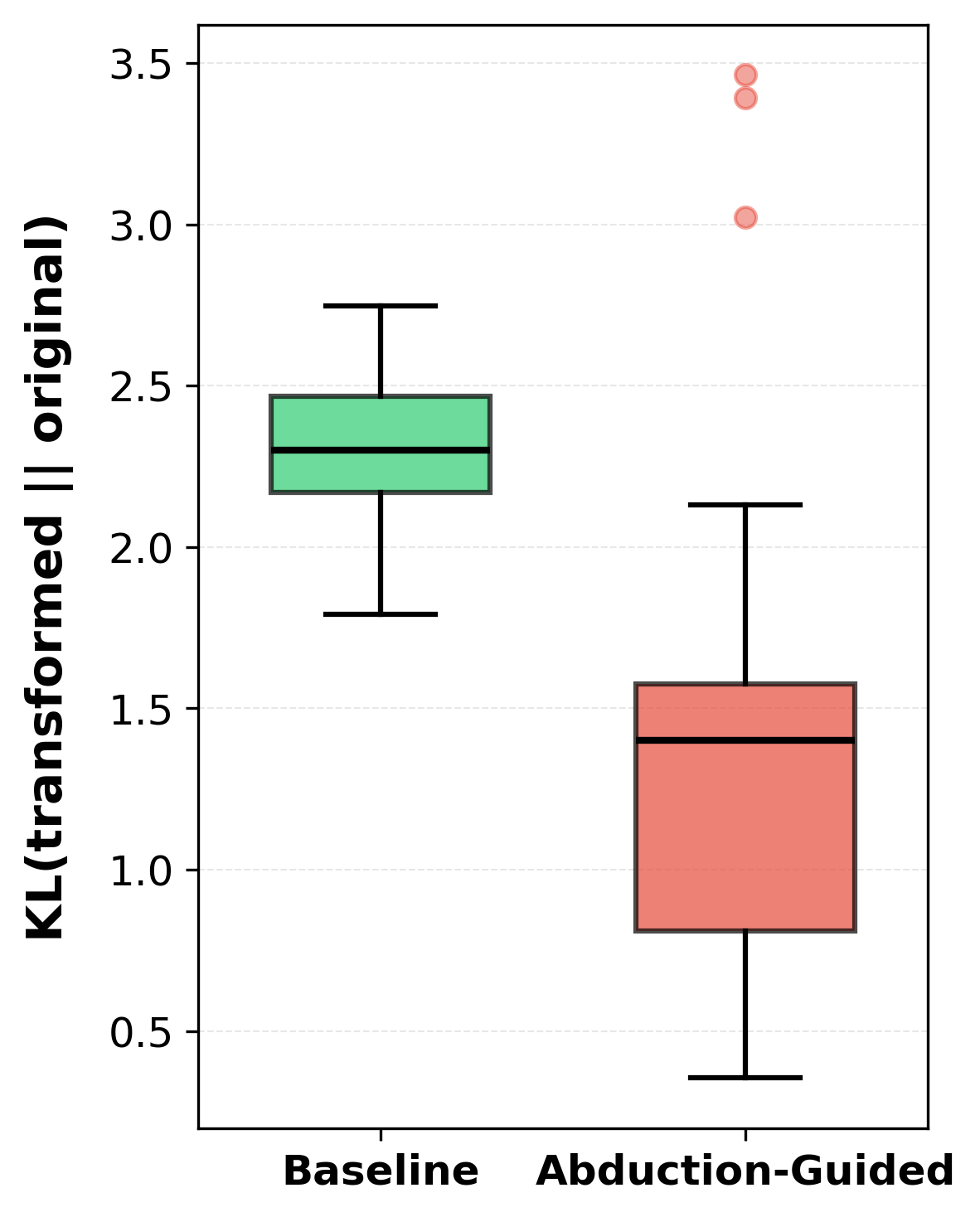}\\[0.0em]
\end{tabular}
&
\begin{tabular}{@{}c@{}}
    \textbf{Llama-4} \\[0em]
        \includegraphics[width=0.47\linewidth]{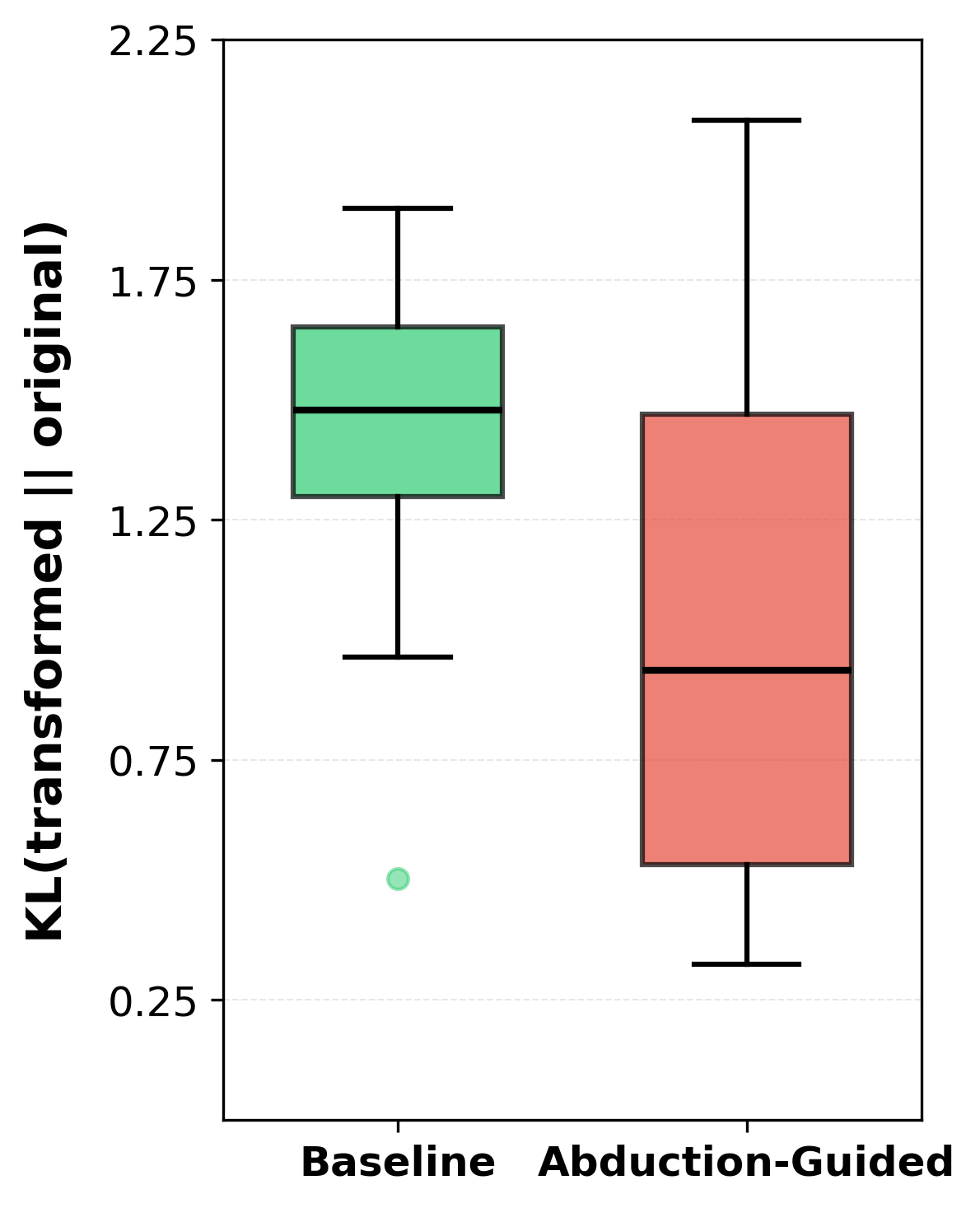}
\end{tabular}

\end{tabular}
\caption{The KL divergence across the stories for I$\rightarrow$C narrative shift. Top left: GPT-4o, Top right: Grok-4, Bottom left: Deepseek-R1, Bottom right: Llama-4.}
\label{fig:kl_comparison_inv}
\end{figure}

On the other hand, for I$\rightarrow$C narrative shift, we observe a $89.64\%$ average improvement across the models over original stories as seen in Table~\ref{tab:inverse_results}. Specifically, GPT outperforms the baseline by $15\%$ while preserving the similarity with the original stories by $36\%$ than the baseline (as seen in Figure \ref{fig:kl_comparison_inv}). Comparable results were observed for Grok and Llama. 
Although Deepseek outperformed our approach by $6\%$, it lost information form the original stories by $39.10\%$ compared to our approach. A study~\cite{deepseekcol} finds that DeepSeek produces more collectivist-oriented content than other LLMs, indicating a stronger collectivistic narrative orientation, which is further supported by its baseline performance in collectivistic story transformation in Table~\ref{tab:inverse_results}.
A qualitative example of I$\rightarrow$C is shown in Figure~\ref{fig:i_to_c_transform} demonstrating how abduction-guided transformation targets specific narrative chunks while preserving narrative structure. Similar C$\rightarrow$I example is provided in Appendix D.

\begin{figure}[ht!]
\begin{tabular}{@{}c@{\hspace{0em}}c@{}}
\begin{tabular}{@{}c@{}}
    \textbf{C $\rightarrow$ I} \\[0em]
        \includegraphics[width=0.49\linewidth]{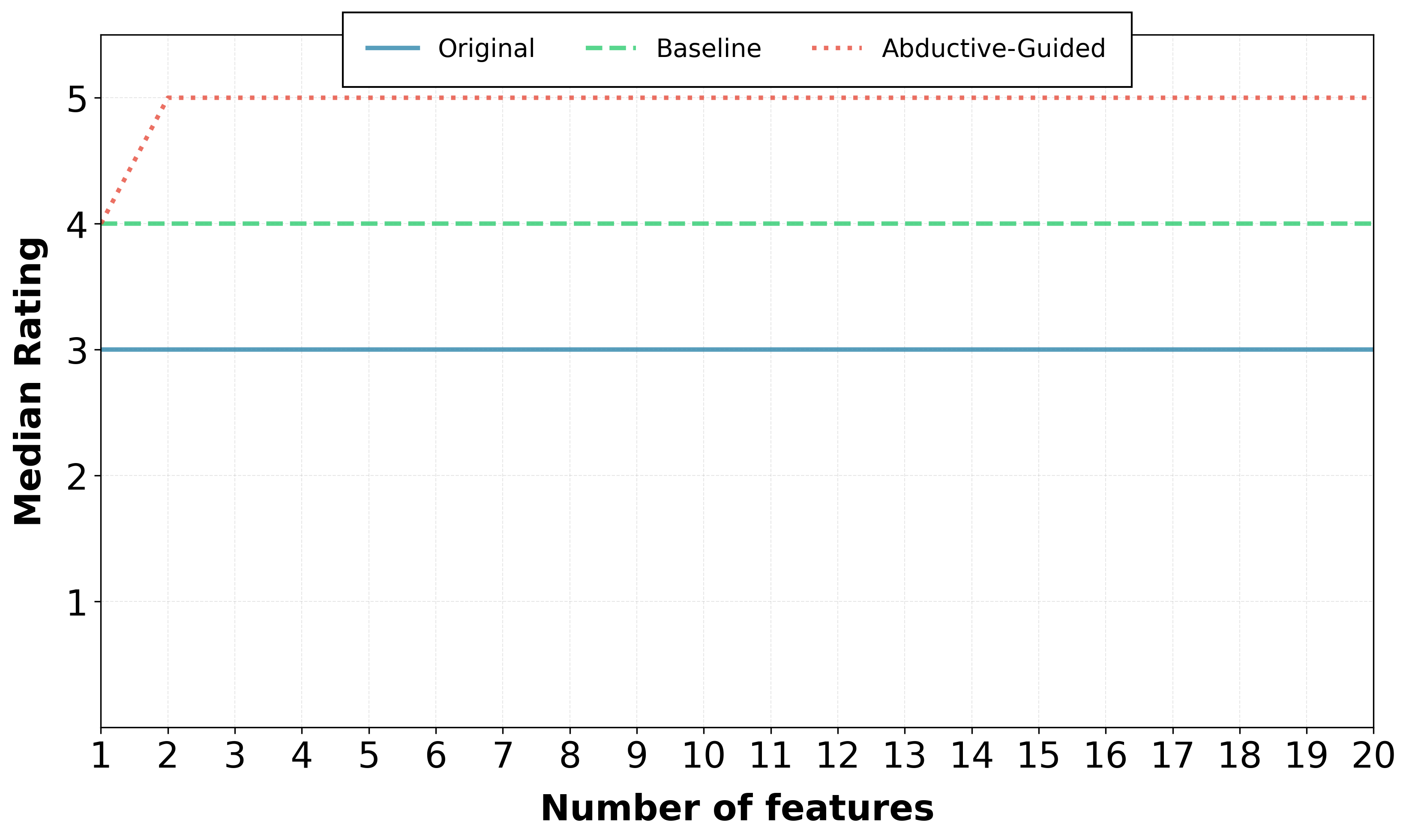}\\[0.0em]
\end{tabular}
&
\begin{tabular}{@{}c@{}}
    \textbf{I $\rightarrow$ C} \\[0em]
        \includegraphics[width=0.49\linewidth]{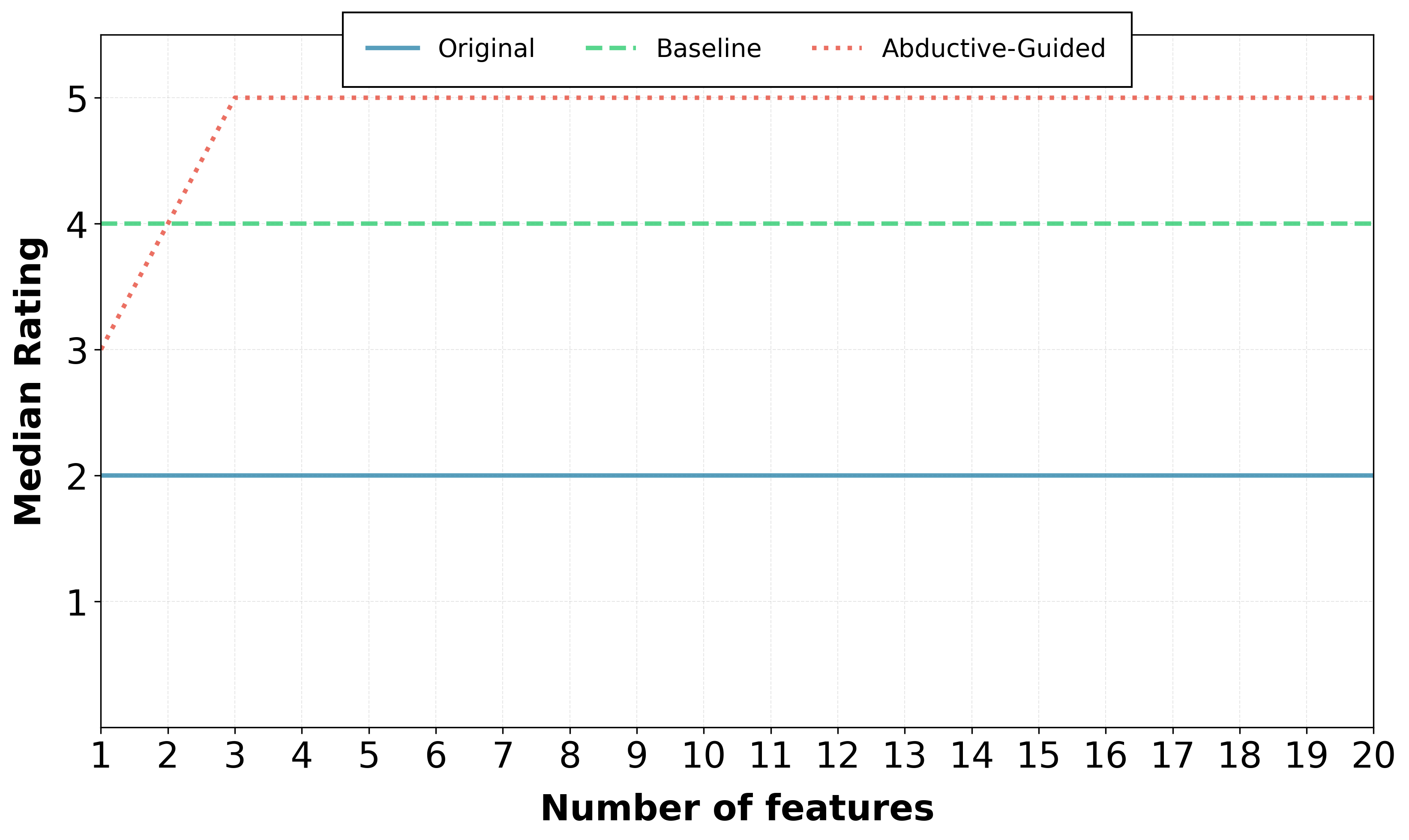}\\[0.0em]
\end{tabular}
\end{tabular}
\caption{Hyperparameter sensitivity. Left: C$\rightarrow$I, Right: I$\rightarrow$C.}
\label{fig:hyperparameter}
\end{figure}

We found that, based on the diagnosis with GPT-4o, for the C$\rightarrow$I shift, 9 stories in the collectivistic corpus were already closer to the individualistic end, while 6 were diagnosed as neutral narratives 
While, for the I$\rightarrow$C, shift, 2 stories were diagnosed as neutral narratives. All of these stories were successfully transformed into the target narrative. Additionally, our abduction-guided approach demonstrates efficiency by targeting only essential narrative segments. 

Overall, for majority models, abduction-guidance improves the baseline for shifting the narrative while maintaining fidelity. Deepseek as a baseline performs competitively with our approach to shift to a collectivistic narrative but loses content similarity from the original story. Moreover, all models perform better at transforming stories into collectivistic than into individualistic narrative.

\begin{figure}
    \centering
    \includegraphics[width=0.8\linewidth]{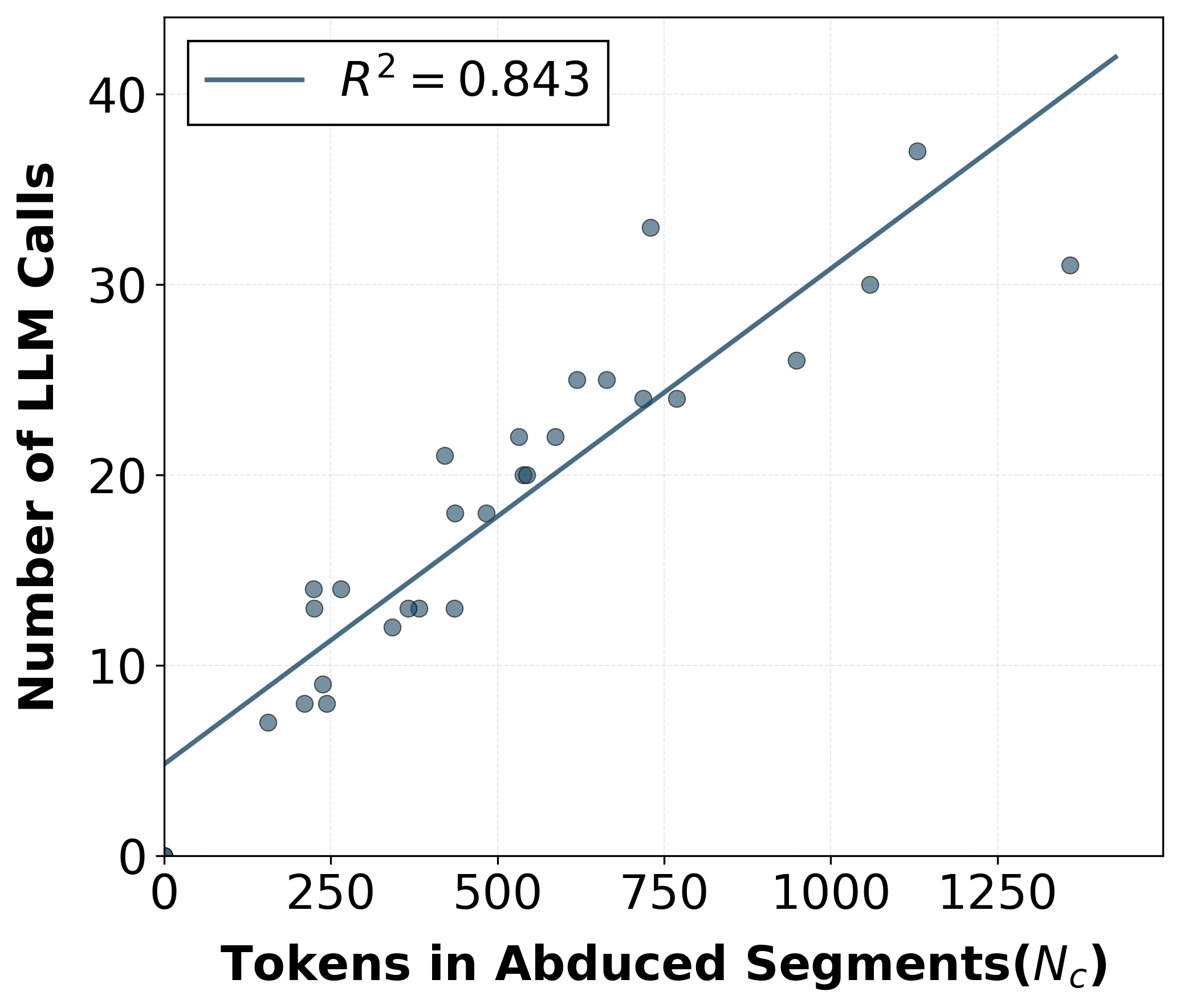}
    \caption{Total tokens of identified segments vs Number of LLM calls for gpt-4o (C$\rightarrow$I).}
    \label{fig:tokens_used}
\end{figure}
\paragraph{Hyperparameter Sensitivity. } We allow for a hyperparameter, number of features, which is the  number of features ($\psi$) in the solution of $\langle \mathcal{O},\mathcal{H},\Pi \rangle$ from Section~\ref{sec:logical_language}.
We report the diagnosis rating across different $\psi$ in Figure~\ref{fig:hyperparameter}. For C$\rightarrow$I, the performance of our approach saturates for $\psi = 2$ at $0.66$ times the original and $0.25$ times the baseline. For I$\rightarrow$C, our approach has a neutral narrative at $\psi = 1$ and shifts toward the target narrative with the increase in $\psi$ while outperforming the baseline before saturating for $\psi = 3$. It achieves the maximum score- $5$ while outperforming the original by $150\%$ and the baseline by $25\%$.

\paragraph{Token usage evaluation.}  
Our approach shows efficiency through targeted chunk modification. Across all stories, only a median of 32.11\% of tokens required transformation, with the remaining content preserved unchanged. Figure~\ref{fig:tokens_used} confirms the theoretical bound of Proposition~\ref{prop:cost}, showing a linear relationship between tokens in identified segments ($N_c$) and number of LLM calls with $R^2 = 0.843$.

\section{Conclusion}\label{sec:conclusion}
Aligning text to an audience's view is significant in journalism, social science and legal communities. While LLMs have been seen to excel in text generation, they seem to lack mechanisms for transforming cultural narrative hidden in the stories while preserving semantic fidelity. Narrative transformation is a high level reasoning task that requires social science domain knowledge, hence we propose a neurosymbolic approach that combines abductive reasoning with LLM based social science diagnostics to shift the narrative of a given story. Our approach shows significant improvement over zero-shot prompting where it either fails to transform the narrative or generate incoherent transformations.  Across both reasoning and non-reasoning models, our abduction-guided approach showed significant improvement over zero-shot baseline. For GPT-4o, our approach outperformed zero-shot narrative transformation by $55.88\%$ while maintaining $40.40\%$ more semantic similarity with the original stories for collectivistic-to-individualistic transformation. This work can be extended to transform the story to align it to a specific cultural audience. Another line of direction is to look at a richer temporal rule structure to inform the abductive transformation.

\section*{Acknowledgments}
This research was funded under the DARPA CONNECT program. Portions of this research relate to U.S. provisional patent application 64/097,607: Neurosymbolic Abductive Reasoning for Reframing Texts (June 24, 2026).
\bibliographystyle{named}
\bibliography{ijcai26}

@String{Computer = "{IEEE} Computer" }

@String{Academic = "Academic Press" }

@article{ks92,
title = {Theory of generalized annotated logic programming and its applications**A preliminary report on this research has appeared in [34].},
journal = {The Journal of Logic Programming},
volume = {12},
number = {4},
pages = {335-367},
year = {1992},
issn = {0743-1066},
doi = {https://doi.org/10.1016/0743-1066(92)90007-P},
url = {https://www.sciencedirect.com/science/article/pii/074310669290007P},
author = {Michael Kifer and V.S. Subrahmanian}
}

@article{Bavikadi_2025,
   title={Geospatial Trajectory Generation via Efficient Abduction: Deployment for Independent Testing},
   volume={416},
   ISSN={2075-2180},
   url={http://dx.doi.org/10.4204/EPTCS.416.24},
   DOI={10.4204/eptcs.416.24},
   journal={Electronic Proceedings in Theoretical Computer Science},
   publisher={Open Publishing Association},
   author={Bavikadi, Divyagna and Aditya, Dyuman and Parkar, Devendra and Shakarian, Paulo and Mueller, Graham and Parvis, Chad and Simari, Gerardo I.},
   year={2025},
   month=feb, pages={274–287} }

@article{Bavikadi2025SeacretAM,
  title={Sea-cret Agents: Maritime Abduction for Region Generation to Expose Dark Vessel Trajectories},
  author={Divyagna Bavikadi and Nathaniel Lee and Paulo Shakarian and Chad Parvis},
  journal={ArXiv},
  year={2025},
  volume={abs/2502.01503},
  url={https://api.semanticscholar.org/CorpusID:276107499}
}

@INPROCEEDINGS {ssTAI22,
author = { Shakarian, Paulo and Simari, Gerardo I. },
booktitle = { 2022 Fourth International Conference on Transdisciplinary AI (TransAI) },
title = {{ Extensions to Generalized Annotated Logic and an Equivalent Neural Architecture }},
year = {2022},
volume = {},
ISSN = {},
pages = {63-70},
doi = {https://doi.org/10.1109/TransAI54797.2022.00017},

publisher = {IEEE Computer Society},
address = {Los Alamitos, CA, USA},
month =sep}

@misc{aditya2023pyreason,
      title={PyReason: Software for Open World Temporal Logic}, 
      author={Dyuman Aditya and Kaustuv Mukherji and Srikar Balasubramanian and Abhiraj Chaudhary and Paulo Shakarian},
      year={2023},
      eprint={2302.13482},
        doi={https://doi.org/10.48550/arXiv.2302.13482},
      archivePrefix={arXiv},
      primaryClass={cs.LO}
}

@book{bercovitch1975puritan,
  title={The Puritan origins of the American self},
  author={Bercovitch, Sacvan},
  year={1975},
  publisher={Yale University Press}
}

@book{klinkowitz2004vonnegut,
  title={The Vonnegut Effect},
  author={Klinkowitz, Jerome},
  year={2004},
  publisher={Univ of South Carolina Press}
}

@incollection{markus2014culture,
  title={Culture and the self: Implications for cognition, emotion, and motivation},
  author={Markus, Hazel Rose and Kitayama, Shinobu},
  booktitle={College student development and academic life},
  pages={264--293},
  year={2014},
  publisher={Routledge}
}

@book{mcadams1993stories,
  title={The stories we live by: Personal myths and the making of the self},
  author={McAdams, Dan P},
  year={1993},
  publisher={Guilford press}
}

@book{moon2020small,
  title={A small boy and others: Imitation and initiation in American culture from Henry James to Andy Warhol},
  author={Moon, Michael},
  year={2020},
  publisher={Duke University Press}
}

@inproceedings{reif2022recipe,
  title={A recipe for arbitrary text style transfer with large language models},
  author={Reif, Emily and Ippolito, Daphne and Yuan, Ann and Coenen, Andy and Callison-Burch, Chris and Wei, Jason},
  booktitle={Proceedings of the 60th Annual Meeting of the Association for Computational Linguistics (Volume 2: Short Papers)},
  pages={837--848},
  year={2022}
}

@article{suzgun2022prompt,
  title={Prompt-and-rerank: A method for zero-shot and few-shot arbitrary textual style transfer with small language models},
  author={Suzgun, Mirac and Melas-Kyriazi, Luke and Jurafsky, Dan},
  journal={arXiv preprint arXiv:2205.11503},
  year={2022}
}

@article{mukherjee2024large,
  title={Are Large Language Models Actually Good at Text Style Transfer?},
  author={Mukherjee, Sourabrata and Ojha, Atul Kr and Dusek, Ondrej},
  journal={CoRR},
  year={2024}
}

@article{deepseekcol,
author = {Segerer, Robin},
year = {2025},
month = {05},
pages = {},
title = {Cultural Value Alignment in Large Language Models: A Prompt-based Analysis of Schwartz Values in Gemini, ChatGPT, and DeepSeek},
doi = {10.31234/osf.io/bmkzw_v1}
}

@article{triandis1995individualism,
  title={Individualism and collectivism West view Press},
  author={Triandis, HC},
  journal={Boulder, CO},
  year={1995}
}

@article{green2000role,
  title={The role of transportation in the persuasiveness of public narratives.},
  author={Green, Melanie C and Brock, Timothy C},
  journal={Journal of personality and social psychology},
  volume={79},
  number={5},
  pages={701},
  year={2000},
  publisher={American Psychological Association}
}

@book{hofstede1984culture,
  title={Culture's consequences: International differences in work-related values},
  author={Hofstede, Geert},
  volume={5},
  year={1984},
  publisher={sage}
}

@article{singelis1994measurement,
  title={The measurement of independent and interdependent self-construals},
  author={Singelis, Theodore M},
  journal={Personality and social psychology bulletin},
  volume={20},
  number={5},
  pages={580--591},
  year={1994},
  publisher={Sage Publications Sage CA: Thousand Oaks, CA}
}

@article{bruner1991narrative,
  title={The narrative construction of reality},
  author={Bruner, Jerome},
  journal={Critical inquiry},
  volume={18},
  number={1},
  pages={1--21},
  year={1991},
  publisher={University of Chicago Press}
}

@article{paul2016russian,
  title={The Russian “firehose of falsehood” propaganda model},
  author={Paul, Christopher and Matthews, Miriam},
  journal={Rand Corporation},
  volume={2},
  number={7},
  pages={1--10},
  year={2016},
  publisher={JSTOR}
}

@article{singelis1995culture,
  title={Culture, self, and collectivist communication: Linking culture to individual behavior},
  author={Singelis, Theodore M and Brown, William J},
  journal={Human communication research},
  volume={21},
  number={3},
  pages={354--389},
  year={1995},
  publisher={Oxford University Press}
}

@book{achebe2012hopes,
  title={Hopes and impediments: Selected essays},
  author={Achebe, Chinua},
  year={2012},
  publisher={Penguin}
}

@article{ba1981living,
  title={The living tradition},
  author={B{\^a}, Amadou Hampat{\'e}},
  journal={General history of Africa},
  volume={1},
  pages={166--205},
  year={1981},
  publisher={Heinemann Unesco}
}

@inproceedings{singh-etal-2025-llms,
    title = "Can {LLM}s Narrate Tabular Data? An Evaluation Framework for Natural Language Representations of Text-to-{SQL} System Outputs",
    author = "Singh, Jyotika  and
      Sun, Weiyi  and
      Agarwal, Amit  and
      Krishnamurthy, Viji  and
      Benajiba, Yassine  and
      Ravi, Sujith  and
      Roth, Dan",
    editor = "Potdar, Saloni  and
      Rojas-Barahona, Lina  and
      Montella, Sebastien",
    booktitle = "Proceedings of the 2025 Conference on Empirical Methods in Natural Language Processing: Industry Track",
    month = nov,
    year = "2025",
    address = "Suzhou (China)",
    publisher = "Association for Computational Linguistics",
    url = "https://aclanthology.org/2025.emnlp-industry.60/",
    doi = "10.18653/v1/2025.emnlp-industry.60",
    pages = "883--902",
    ISBN = "979-8-89176-333-3"
}

@inproceedings{mukherjee2025evaluating,
  title={Evaluating Text Style Transfer Evaluation: Are There Any Reliable Metrics?},
  author={Mukherjee, Sourabrata and Ojha, Atul Kr and McCrae, John Philip and Du{\v{s}}ek, Ond{\v{r}}ej},
  booktitle={Proceedings of the 2025 Conference of the Nations of the Americas Chapter of the Association for Computational Linguistics: Human Language Technologies (Volume 4: Student Research Workshop)},
  pages={418--434},
  year={2025}
}

@article{yue2025relate,
  title={Relate-sim: Leveraging turning point theory and llm agents to predict and understand long-term relationship dynamics through interactive narrative simulations},
  author={Yue, Matthew and Xu, Zhikun and Gupta, Vivek and Ha, Thao and Sharabi, Liesal and Zhou, Ben},
  journal={arXiv preprint arXiv:2510.00414},
  year={2025}
}

@inproceedings{chun2025conflictlens,
  title={Conflictlens: Llm-based conflict resolution training in romantic relationship},
  author={Chun, Jiwon and Zhang, Gefei and Xia, Meng},
  booktitle={Adjunct Proceedings of the 38th Annual ACM Symposium on User Interface Software and Technology},
  pages={1--3},
  year={2025}
}

@inproceedings{park2023generative,
  title={Generative agents: Interactive simulacra of human behavior},
  author={Park, Joon Sung and O'Brien, Joseph and Cai, Carrie Jun and Morris, Meredith Ringel and Liang, Percy and Bernstein, Michael S},
  booktitle={Proceedings of the 36th annual acm symposium on user interface software and technology},
  pages={1--22},
  year={2023}
}

@inproceedings{tian2024large,
  title={Are Large Language Models Capable of Generating Human-Level Narratives?},
  author={Tian, Yufei and Huang, Tenghao and Liu, Miri and Jiang, Derek and Spangher, Alexander and Chen, Muhao and May, Jonathan and Peng, Nanyun},
  booktitle={Proceedings of the 2024 Conference on Empirical Methods in Natural Language Processing},
  pages={17659--17681},
  year={2024}
}

@inproceedings{pei2024swag,
  title={SWAG: Storytelling with action guidance},
  author={Pei, Jonathan and Patel, Zeeshan and El-Refai, Karim and Li, Tianle},
  booktitle={Findings of the Association for Computational Linguistics: EMNLP 2024},
  pages={14086--14106},
  year={2024}
}

@inproceedings{piper-bagga-2024-using,
    title = "Using Large Language Models for Understanding Narrative Discourse",
    author = "Piper, Andrew  and
      Bagga, Sunyam",
    editor = "Lal, Yash Kumar  and
      Clark, Elizabeth  and
      Iyyer, Mohit  and
      Chaturvedi, Snigdha  and
      Brei, Anneliese  and
      Brahman, Faeze  and
      Chandu, Khyathi Raghavi",
    booktitle = "Proceedings of the 6th Workshop on Narrative Understanding",
    month = nov,
    year = "2024",
    address = "Miami, Florida, USA",
    publisher = "Association for Computational Linguistics",
    url = "https://aclanthology.org/2024.wnu-1.4/",
    doi = "10.18653/v1/2024.wnu-1.4",
    pages = "37--46"
}

@inproceedings{li-etal-2024-search,
    title = "In Search of the Long-Tail: Systematic Generation of Long-Tail Inferential Knowledge via Logical Rule Guided Search",
    author = "Li, Huihan  and
      Ning, Yuting  and
      Liao, Zeyi  and
      Wang, Siyuan  and
      Li, Xiang Lorraine  and
      Lu, Ximing  and
      Zhao, Wenting  and
      Brahman, Faeze  and
      Choi, Yejin  and
      Ren, Xiang",
    editor = "Al-Onaizan, Yaser  and
      Bansal, Mohit  and
      Chen, Yun-Nung",
    booktitle = "Proceedings of the 2024 Conference on Empirical Methods in Natural Language Processing",
    month = nov,
    year = "2024",
    address = "Miami, Florida, USA",
    publisher = "Association for Computational Linguistics",
    url = "https://aclanthology.org/2024.emnlp-main.140/",
    doi = "10.18653/v1/2024.emnlp-main.140",
    pages = "2348--2370"
}

@inproceedings{javadi2025can,
  title={Can stories help LLMs reason? curating information space through narrative},
  author={Javadi, Vahid Sadiri and Trippas, Johanne R and Lal, Yash Kumar and Flek, Lucie},
  booktitle={The 2nd Workshop on Analogical Abstraction in Cognition, Perception, and Language (Analogy-Angle II)},
  pages={92},
  year={2025}
}

@article{choi2025individualism,
  title={Individualism--collectivism and intergroup bias},
  author={Choi, Hoon-Seok},
  journal={Asian Journal of Social Psychology},
  volume={28},
  number={2},
  pages={e70010},
  year={2025},
  publisher={Wiley Online Library}
}

@inproceedings{rooein2025biased,
  title={Biased tales: Cultural and topic bias in generating children’s stories},
  author={Rooein, Donya and Zouhar, Vil{\'e}m and Nozza, Debora and Hovy, Dirk},
  booktitle={Proceedings of the 2025 Conference on Empirical Methods in Natural Language Processing},
  pages={52--72},
  year={2025}
}

@article{jenner2025using,
  title={Using large language models for narrative analysis: a novel application of generative AI},
  author={Jenner, Sarah and Raidos, Dimitris and Anderson, Emma and Fleetwood, Stella and Ainsworth, Ben and Fox, Kerry and Kreppner, Jana and Barker, Mary},
  journal={Methods in Psychology},
  volume={12},
  pages={100183},
  year={2025},
  publisher={Elsevier}
}

@article{tao2025cat,
  title={CAT-LLM: Style-enhanced Large Language Models with Text Style Definition for Chinese Article-style Transfer},
  author={Tao, Zhen and Xi, Dinghao and Li, Zhiyu and Tang, Liumin and Xu, Wei},
  journal={ACM Transactions on Knowledge Discovery from Data},
  volume={19},
  number={7},
  pages={1--33},
  year={2025},
  publisher={ACM New York, NY}
}

@incollection{GREEN20241,
title = {Chapter One - Narrative transportation: How stories shape how we see ourselves and the world},
editor = {Bertram Gawronski},
series = {Advances in Experimental Social Psychology},
publisher = {Academic Press},
volume = {70},
pages = {1-82},
year = {2024},
issn = {0065-2601},
doi = {https://doi.org/10.1016/bs.aesp.2024.03.002},
url = {https://www.sciencedirect.com/science/article/pii/S0065260124000145},
author = {Melanie C. Green and Markus Appel}
}
 \appendix
\newpage

\section{Formal Proof of Runtime Result} \label{app:pf}
We show theoretically the bound on the number of LLM calls with Proposition~\ref{prop:cost_appendix} and that it increases linearly with the number of abduced story chunks.

\begin{proposition}\label{prop:cost_appendix}
    Let $s$ be a story represented as a sequence of tokens and $\chunkStory$ denote 
the set of chunks from $s$. For each chunk $c \in \chunkStory$, 
let $\text{size}(c)$ denote the number of tokens in $c$.  
Then total number of LLM transformations is bounded by following quantity.
\[
N_c = \sum_{c \in \chunkStory} \text{size}(c)
\]


\end{proposition}
\begin{proof}
    We define k as the number of LLM transformation calls in a single iteration of abduction-guided approach. Assume BWOC that $k > N_c$. 

We define $N'$ as the total tokens in story $s$:

\[
N' = N_c + N_{\overline{c}}
\]

where $N_{\overline{c}}$ denotes tokens not in any identified chunk. 
\[
N_c = N' - N_{\overline{c}}
\]
By our assumption, $k > N_c$,
\[
k > N' - N_{\overline{c}}
\]
Adding $N_{\overline{c}}$ on both sides,
\[
k + N_{\overline{c}} > N'
\]
However, by definition $N'$ is the maximum possible tokens in the story and thus $k + N_{\overline{c}}$ (LLM calls in addition to non-identified 
tokens) cannot exceed $N'$. This is a contradiction. Therefore, $k \leq N_c$.
\end{proof}

\section{Narrative Diagnostic Survey}\label{app:surv}
 We provide complete set of individualistic narrative based diagnostic questions in Figure~\ref{fig:individualistic_diagnostic_questions} and collectivistic narrative based diagnostic questions in Figure~\ref{fig:collectivistic_diagnostic_questions}. 

\begin{figure*}[h!]
    \centering
    \fbox{
    \begin{minipage}{0.95\textwidth}
        \small
        
        \textbf{Feature 1: Protagonist-Centered Focus}
        
        ``Is the narrative driven primarily by one character's experiences, choices, and inner life?''
        
        \vspace{0.1cm}
        
        \textbf{Feature 2: Internal Goals}
        
        ``Does the protagonist pursue personal ambitions (e.g., self-actualization, fame, self-expression) over collective objectives?''
        
        \vspace{0.1cm}
        
        \textbf{Feature 3: Decision-Driven Plot}
        
        ``Are key turning points in the story determined by the protagonist's own decisions rather than group mandates or fate?''
        
        \vspace{0.1cm}
        
        \textbf{Feature 4: Self-Reliance}
        
        ``Does the protagonist overcome obstacles through their own resourcefulness rather than relying on communal support?''
        
        \vspace{0.1cm}
        
        \textbf{Feature 5: Individual Accolades}
        
        ``Are awards, status, or recognition attributed primarily to the single protagonist rather than to a team or ensemble?''
        
        \vspace{0.1cm}
        
        \textbf{Feature 6: Meritocracy Emphasis}
        
        ``Is success portrayed as earned by the protagonist's talent, hard work, or innate brilliance, rather than by lineage or group standing?''
        
        \vspace{0.1cm}
        
        \textbf{Feature 7: ``Man vs. Self/World'' Conflict}
        
        ``Is the central struggle internal (e.g., self doubt, identity) or between the protagonist and external forces, rather than group conflicts?''
        
        \vspace{0.1cm}
        
        \textbf{Feature 8: Solo Confrontations}
        
        ``Do climactic showdowns feature the lone protagonist facing the antagonist or obstacle, rather than a collaborative effort?''
        
        \vspace{0.1cm}
        
        \textbf{Feature 9: Inner Journey}
        
        ``Is the character arc centered on the protagonist discovering their own values, strengths, or purpose?''
        
        \vspace{0.1cm}
        
        \textbf{Feature 10: Uniqueness \& Self-Expression}
        
        ``Are characters celebrated for what makes them unique (quirks, dreams) or for `being true to themselves'?''
        
        \vspace{0.1cm}
        
        \textbf{Feature 11: Self-Construal}
        
        ``Does the narrative present the self as stable and independent, defined by personal traits rather than social roles?''
        
        \vspace{0.1cm}
        
        \textbf{Feature 12: Behavioral Guidance}
        
        ``Are actions in the story guided by the protagonist's personal attitudes and preferences rather than social norms or group expectations?''
        
        \vspace{0.1cm}
        
        \textbf{Feature 13: Relationship Orientation}
        
        ``Does the story depict relationships as optional and based on mutual benefit rather than duty and loyalty?''
        
        \vspace{0.1cm}
        
        \textbf{Feature 14: Primary Conflict}
        
        ``Is the central conflict about asserting one's individual identity or resisting conformity?''
        
        \vspace{0.1cm}
        
        \textbf{Feature 15: Resolution Style}
        
        ``Does the story resolve conflicts by standing up for personal rights and achieving justice rather than through compromise and reconciliation?''
        
        \vspace{0.1cm}
        
        \textbf{Feature 16: Moral Emphasis}
        
        ``Does the narrative emphasize autonomy, personal integrity, or self-actualization as moral virtues?''
        
        \vspace{0.1cm}
        
        \textbf{Feature 17: Relationship Framing}
        
        ``Does the story frame relationships as non-essential, allowing the protagonist to pursue goals independently?''
        
        \vspace{0.1cm}
        
        \textbf{Feature 18: Vertical Individualism}
        
        ``Does the story accept social inequality as a natural consequence of individual achievement?''
        
        \vspace{0.1cm}
        
        \textbf{Feature 19: Personal Ethics over Group Norms}
        
        ``Does the protagonist's personal code or conscience take precedence over cultural or familial expectations?''
        
        \vspace{0.1cm}
        
        \textbf{Feature 20: Self-Actualization Climax}
        
        ``Does the emotional payoff come from the protagonist's personal breakthrough instead of restoring group harmony?''
        
    \end{minipage}
    }
    \caption{Complete individualistic diagnostic questions: 20 questions assessing individualistic narrative orientation.}
    \label{fig:individualistic_diagnostic_questions}
\end{figure*}

\begin{figure*}[h!]
    \centering
    \fbox{
    \begin{minipage}{0.95\textwidth}
        \small
        
        \textbf{Feature 1: Protagonist-Centered Focus}
        
        ``Is the narrative driven primarily by the group's or community's shared experiences, collective choices, and communal identity?''
        
        \vspace{0.1cm}
        
        \textbf{Feature 2: Internal Goals}
        
        ``Do the characters pursue group ambitions (e.g., community well-being, family honor, shared success) over individual goals?''
        
        \vspace{0.1cm}
        
        \textbf{Feature 3: Decision-Driven Plot}
        
        ``Are key turning points in the story determined by collective decisions, group mandates, or community traditions rather than by a single individual's decision?''
        
        \vspace{0.1cm}
        
        \textbf{Feature 4: Self-Reliance}
        
        ``Do the characters overcome obstacles through communal support, collective action, or shared resources rather than individual effort?''
        
        \vspace{0.1cm}
        
        \textbf{Feature 5: Individual Accolades}
        
        ``Are awards, status, or recognition attributed primarily to the group, ensemble, or community effort rather than to an individual?''
        
        \vspace{0.1cm}
        
        \textbf{Feature 6: Meritocracy Emphasis}
        
        ``Is success portrayed as resulting from group support, family lineage, or communal contributions rather than solely individual talent?''
        
        \vspace{0.1cm}
        
        \textbf{Feature 7: ``Man vs. Self/World'' Conflict}
        
        ``Is the central struggle between the group and external forces or within group cohesion, rather than an individual's internal conflict?''
        
        \vspace{0.1cm}
        
        \textbf{Feature 8: Solo Confrontations}
        
        ``Do climactic showdowns feature collaborative group efforts or collective confrontation, rather than a lone individual?''
        
        \vspace{0.1cm}
        
        \textbf{Feature 9: Inner Journey}
        
        ``Is the character arc centered on the group discovering shared values, collective strengths, or communal purpose?''
        
        \vspace{0.1cm}
        
        \textbf{Feature 10: Uniqueness \& Self-Expression}
        
        ``Are characters celebrated for conforming to group norms, fulfilling social roles, or contributing to the collective identity rather than uniqueness?''
        
        \vspace{0.1cm}
        
        \textbf{Feature 11: Self-Construal}
        
        ``Does the narrative present the self as interdependent, defined by social roles and relationships rather than solely personal traits?''
        
        \vspace{0.1cm}
        
        \textbf{Feature 12: Behavioral Guidance}
        
        ``Are actions in the story guided by social norms, group expectations, or communal values rather than personal preferences?''
        
        \vspace{0.1cm}
        
        \textbf{Feature 13: Relationship Orientation}
        
        ``Does the story depict relationships as based on duty, loyalty, obligation, and collective well-being rather than solely mutual benefit?''
        
        \vspace{0.1cm}
        
        \textbf{Feature 14: Primary Conflict}
        
        ``Is the central conflict about maintaining social harmony, fulfilling collective roles, or adhering to group norms rather than asserting individuality?''
        
        \vspace{0.1cm}
        
        \textbf{Feature 15: Resolution Style}
        
        ``Does the story resolve conflicts through compromise, reconciliation, and restoring group harmony rather than individual vindication?''
        
        \vspace{0.1cm}
        
        \textbf{Feature 16: Moral Emphasis}
        
        ``Does the narrative emphasize group solidarity, communal responsibility, or collective welfare as moral virtues?''
        
        \vspace{0.1cm}
        
        \textbf{Feature 17: Relationship Framing}
        
        ``Does the story frame relationships as essential, requiring individuals to consider the impact of their actions on family and community?''
        
        \vspace{0.1cm}
        
        \textbf{Feature 18: Vertical Individualism}
        
        ``Does the story emphasize equality, shared prosperity, and collective welfare over social inequality and individual hierarchy?''
        
        \vspace{0.1cm}
        
        \textbf{Feature 19: Personal Ethics over Group Norms}
        
        ``Does adherence to cultural norms, family expectations, or communal codes take precedence over personal preferences?''
        
        \vspace{0.1cm}
        
        \textbf{Feature 20: Self-Actualization Climax}
        
        ``Does the emotional payoff come from restoring group harmony, collective well-being, or communal success rather than an individual breakthrough?''
        
    \end{minipage}
    }
    \caption{Complete collectivistic diagnostic questions: 20 questions assessing collectivistic narrative orientation.}
    \label{fig:collectivistic_diagnostic_questions}
\end{figure*}

\section{Details on Systems Prompts} \label{app:pmts}

We present the system prompt used for transforming the story with the abduced changes.The LLM prompt to transform particular identified story chunk given the current narrative and target narrative is provided in the Figure~\ref{fig:transformation_prompt}.

\begin{figure*}[ht!]
\centering
\begin{tcolorbox}[
    colback=gray!5,
    colframe=black!75,
    boxrule=0.5pt,
    arc=2pt,
    left=6pt,
    right=6pt,
    top=6pt,
    bottom=6pt,
    fontupper=\small\ttfamily
]
\textbf{System:} You are given a \textcolor{blue}{[source\_type]} story below.

\vspace{0.3em}
\textbf{** start of story **}

\vspace{0.2em}
\textcolor{blue}{[story]}

\vspace{0.2em}
\textbf{** end of story **}

\vspace{0.3em}
Your goal is to make the story more \textcolor{blue}{[target\_type]}. To make it more \textcolor{blue}{[target\_type]}, you will update only the selected segment from the story which is provided below.

\vspace{0.3em}
\textbf{Selected segment:} ** \textcolor{blue}{[segment]} **

\vspace{0.3em}
Now, rewrite the whole story by updating only the selected segment to make the story more \textcolor{blue}{[target\_type]}. Don't change other parts of the story, and just output the rewritten story, nothing else.
\end{tcolorbox}
\caption{Transformation prompt template. Placeholders \textcolor{blue}{[source\_type]}, \textcolor{blue}{[target\_type]}, \textcolor{blue}{[story]}, and \textcolor{blue}{[segment]} are instantiated for each identified segment.}
\label{fig:transformation_prompt}
\end{figure*}

\section{Additional Experimental Results}\label{app:res}
In figure~\ref{fig:med_diag_comparison}, \ref{fig:med_diag_comparison_grok4}, \ref{fig:med_diag_comparison_deepseek}, \ref{fig:med_diag_comparison_llama}, we report the median of diagnosis scores and the blue solid bar shows the diagnosis scores, which forms a lower envelope indicating the narrative of the original stories to be on the collectivistic (left column) or individualistic (right column) end. Similarly, \ref{fig:mode_diag_comparison}, \ref{fig:mode_diag_comparison_grok4}, \ref{fig:mode_diag_comparison_deepseek}, \ref{fig:mode_diag_comparison_llama} shows the mode and 
\ref{fig:mean_diag_comparison}, \ref{fig:mean_diag_comparison_grok4}, \ref{fig:mean_diag_comparison_deepseek}, \ref{fig:mean_diag_comparison_llama} median of the diagnosis scores for all of the LLMs.

A qualitative example of C$\rightarrow$I is shown in Figure~\ref{fig:c_i_transform} demonstrating how abduction-guided transformation targets specific narrative chunks while preserving narrative structure.

\begin{figure*}[t!]
\begin{tabular}{@{}c@{\hspace{0em}}c@{}}
\begin{tabular}{@{}c@{}}
    \textbf{C$\rightarrow$I} \\[0em]
        \includegraphics[width=0.49\linewidth]{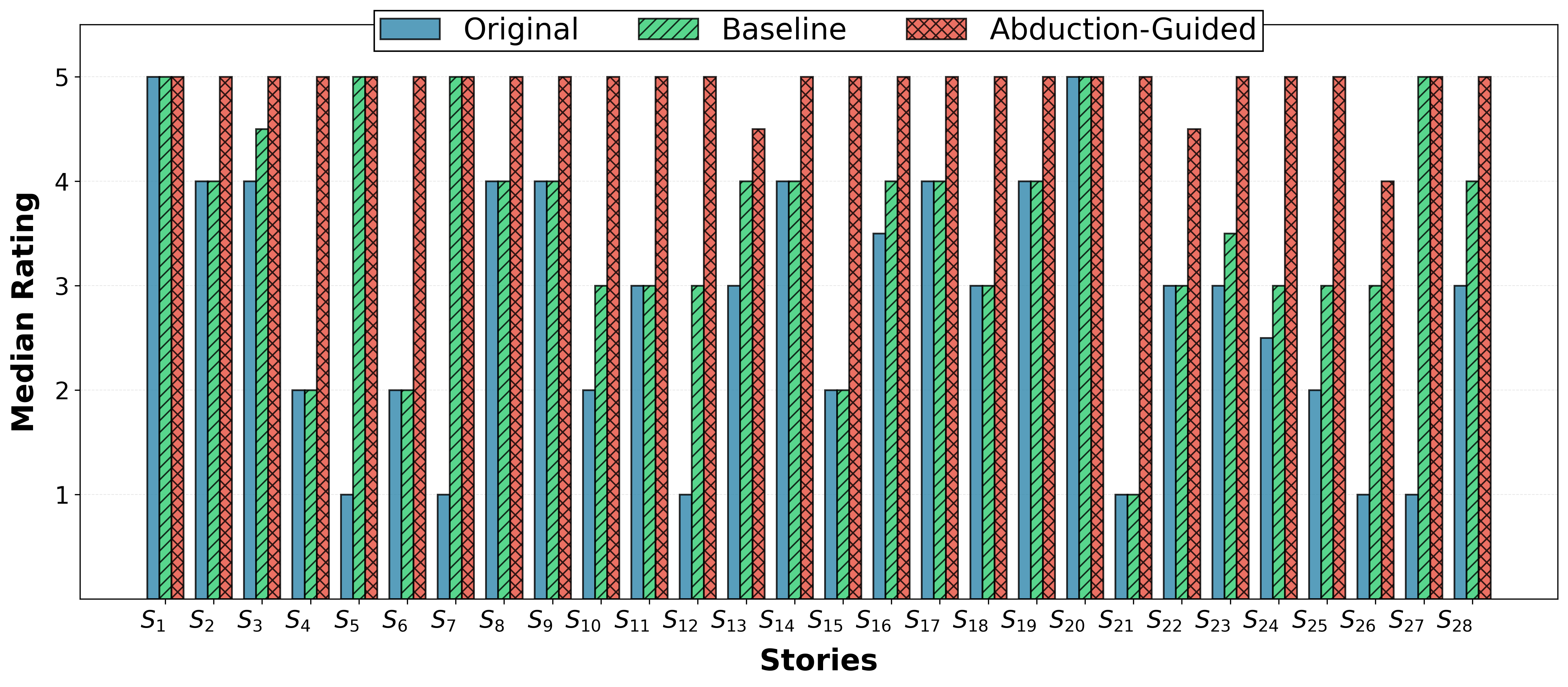}\\[0.0em]
\end{tabular}
&
\begin{tabular}{@{}c@{}}
    \textbf{I$\rightarrow$C} \\[0em]
        \includegraphics[width=0.49\linewidth]{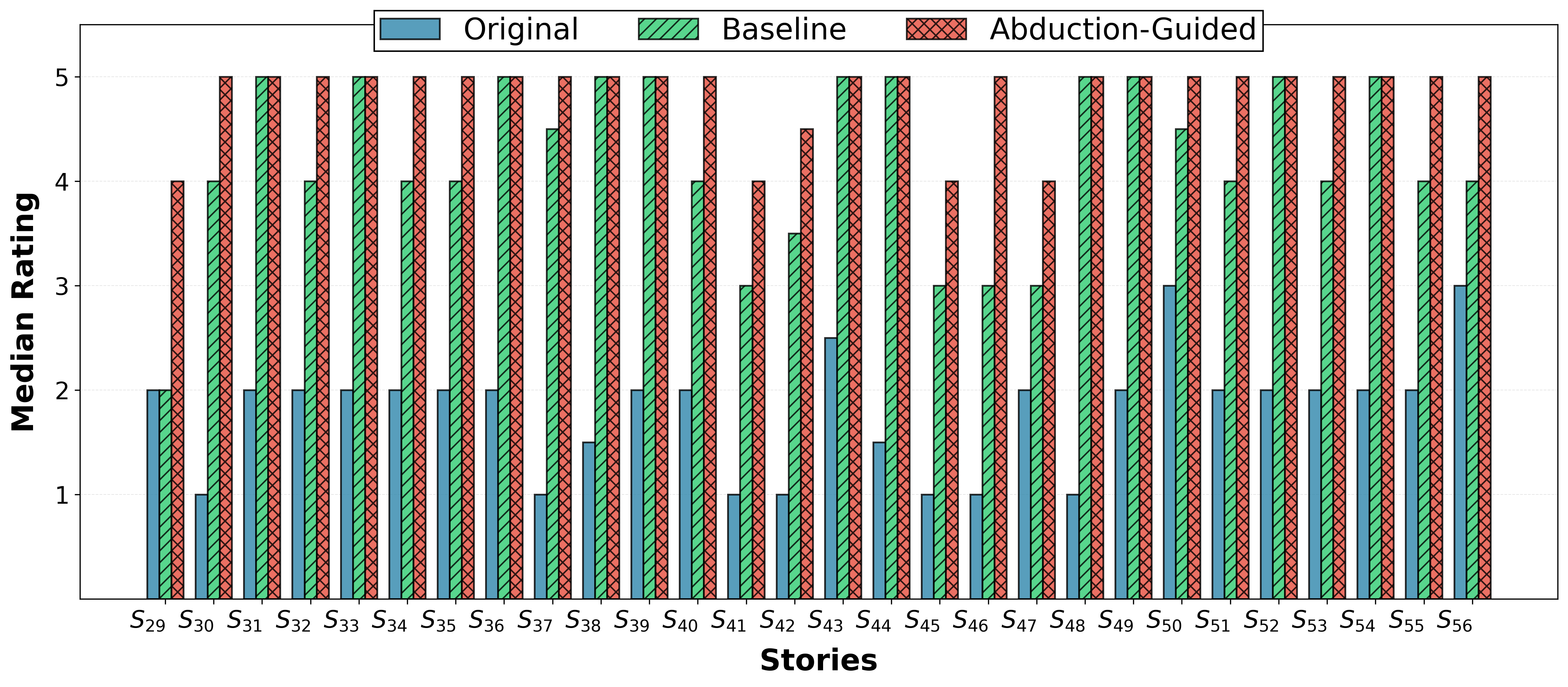}\\[0.0em]
\end{tabular}
\end{tabular}
\caption{The median of diagnosis ratings for GPT-4o across the stories in the dataset. Left: C$\rightarrow$I, Right: I$\rightarrow$C. Note that for both directions higher rating is better.}
\label{fig:med_diag_comparison}
\end{figure*}
\begin{figure*}[!t]
\begin{tabular}{@{}c@{\hspace{0em}}c@{}}
\begin{tabular}{@{}c@{}}
    \textbf{C$\rightarrow$I} \\[0em]
        \includegraphics[width=0.49\linewidth]{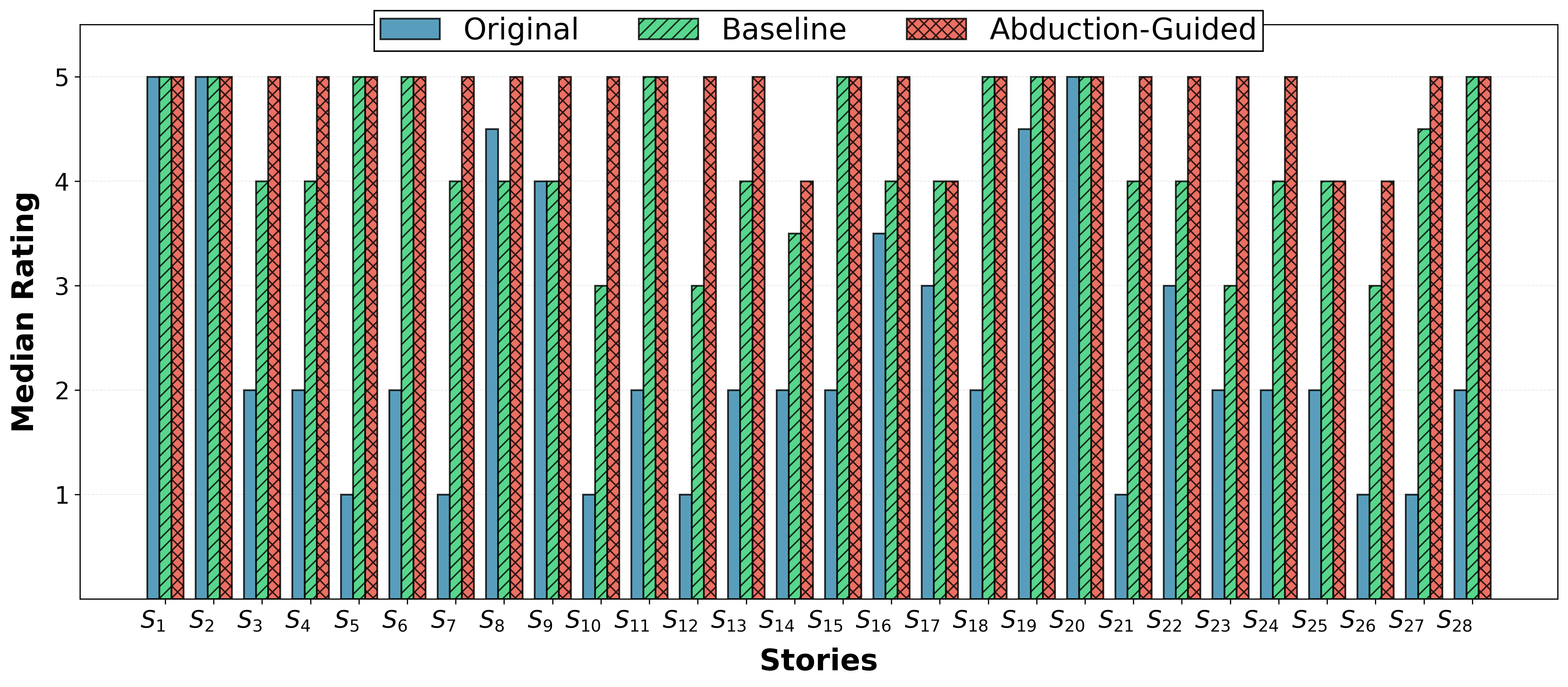}\\[0.0em]
\end{tabular}
&
\begin{tabular}{@{}c@{}}
    \textbf{I$\rightarrow$C} \\[0em]
        \includegraphics[width=0.49\linewidth]{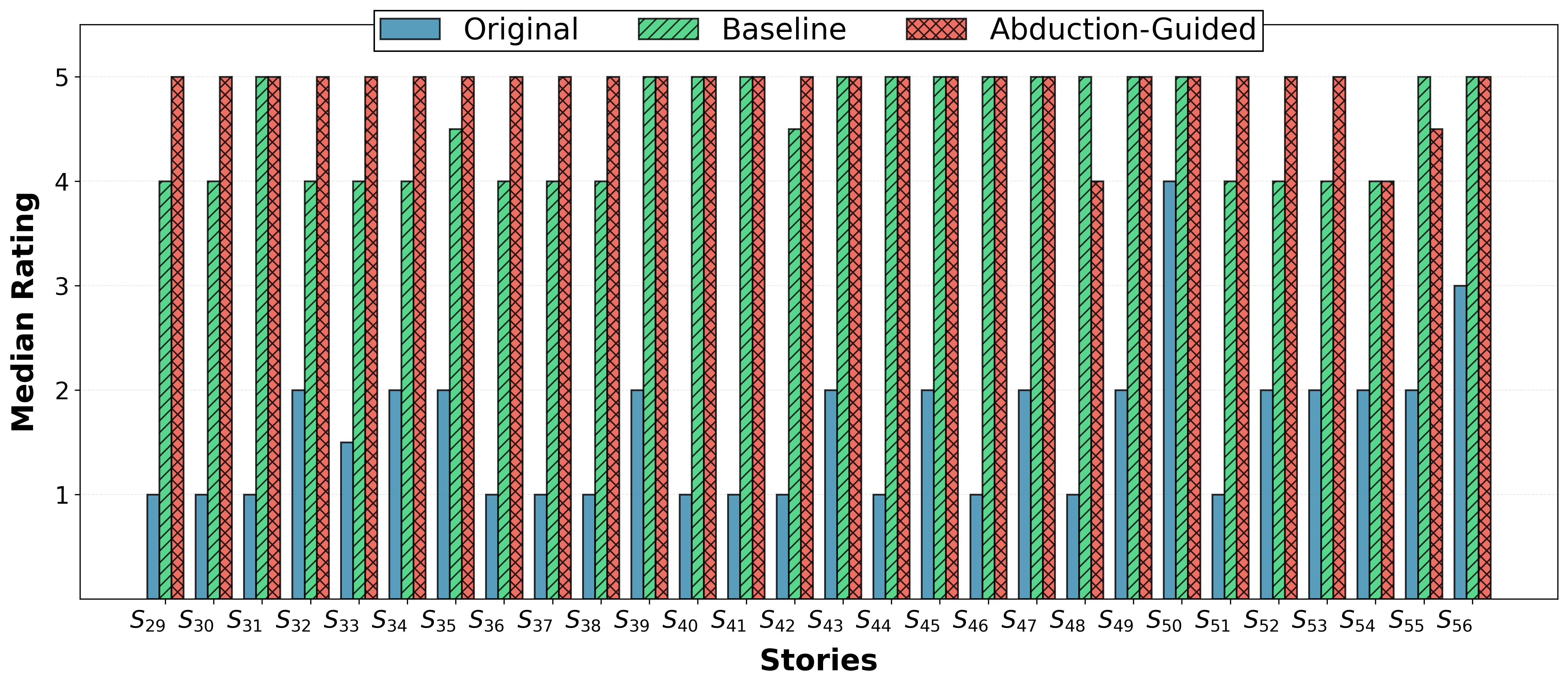}\\[0.0em]
\end{tabular}
\end{tabular}
\caption{The median of diagnosis ratings for Grok-4 across the stories in the dataset. Left: C$\rightarrow$I, Right: I$\rightarrow$C. Note that for both directions higher rating is better.}
\label{fig:med_diag_comparison_grok4}
\end{figure*}

\begin{figure*}[!t]
\begin{tabular}{@{}c@{\hspace{0em}}c@{}}
\begin{tabular}{@{}c@{}}
    \textbf{C$\rightarrow$I} \\[0em]
        \includegraphics[width=0.49\linewidth]{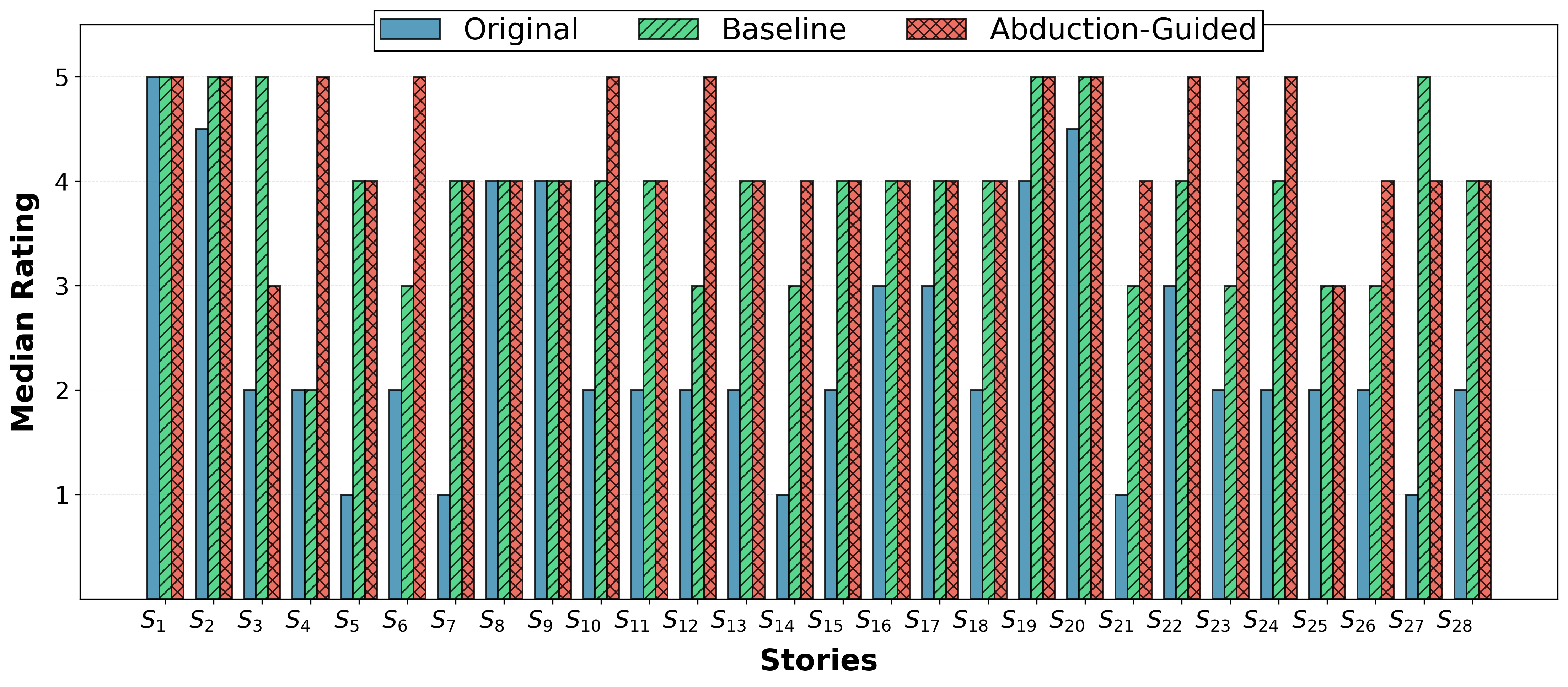}\\[0.0em]
\end{tabular}
&
\begin{tabular}{@{}c@{}}
    \textbf{I$\rightarrow$C} \\[0em]
        \includegraphics[width=0.49\linewidth]{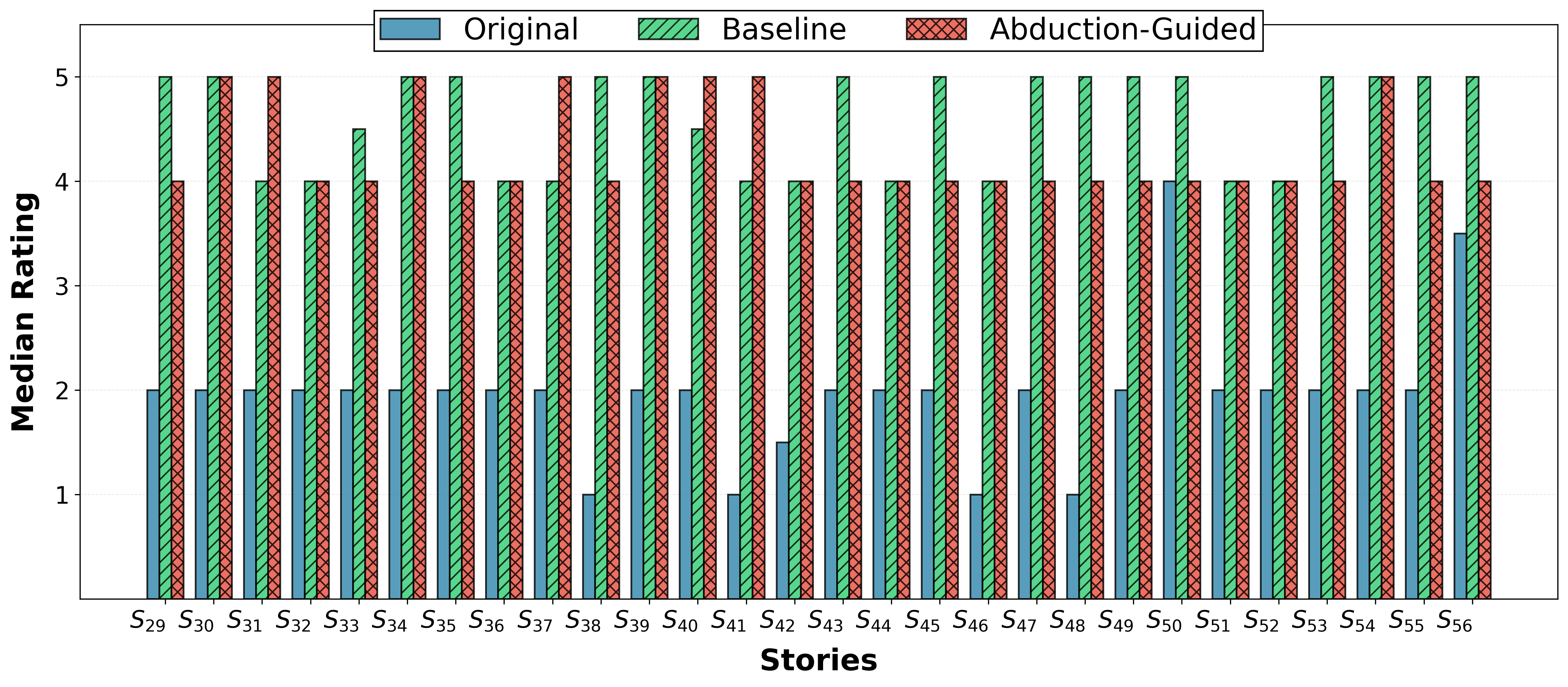}\\[0.0em]
\end{tabular}
\end{tabular}
\caption{The median of diagnosis ratings for Deepseek-R1 across the stories in the dataset. Left: C$\rightarrow$I, Right: I$\rightarrow$C. Note that for both directions higher rating is better.}
\label{fig:med_diag_comparison_deepseek}
\end{figure*}

\begin{figure*}[!t]
\begin{tabular}{@{}c@{\hspace{0em}}c@{}}
\begin{tabular}{@{}c@{}}
    \textbf{C$\rightarrow$I} \\[0em]
        \includegraphics[width=0.49\linewidth]{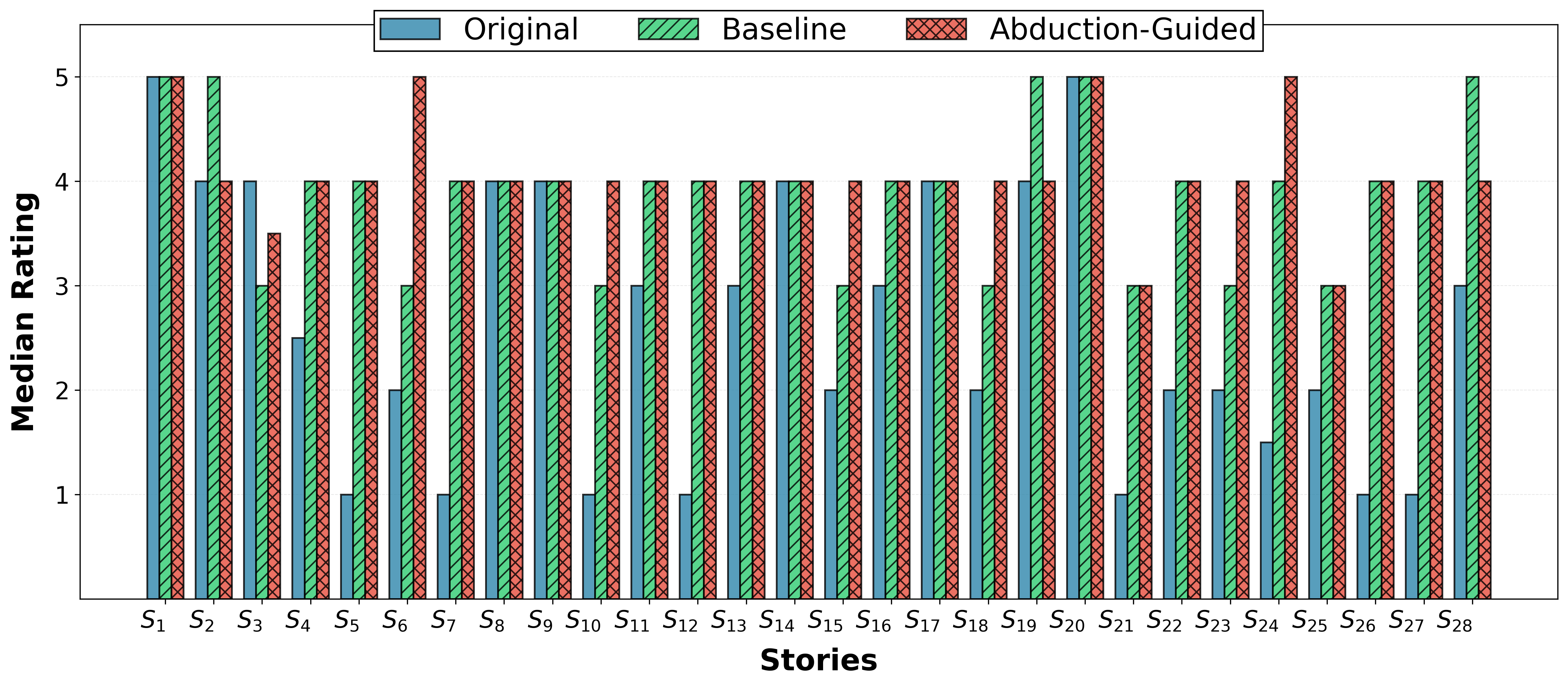}\\[0.0em]
\end{tabular}
&
\begin{tabular}{@{}c@{}}
    \textbf{I$\rightarrow$C} \\[0em]
        \includegraphics[width=0.49\linewidth]{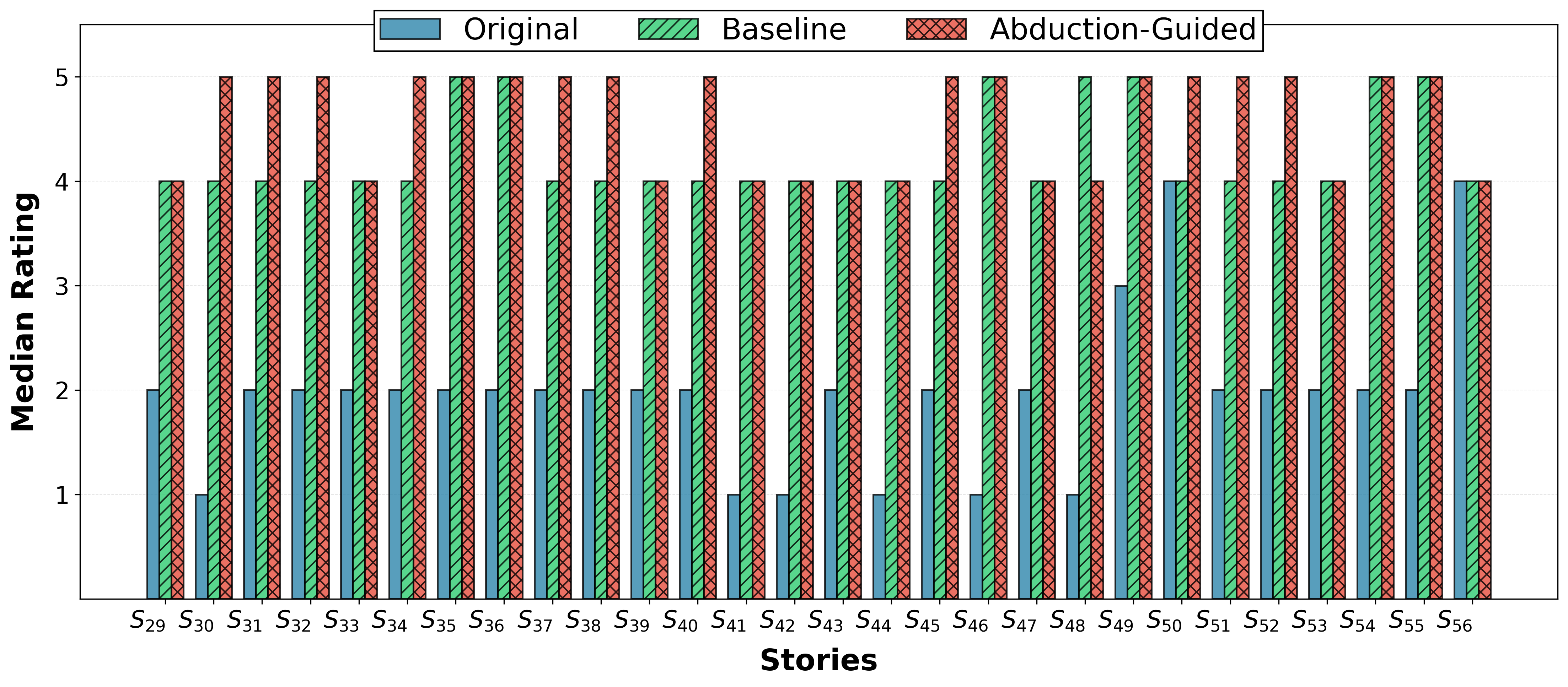}\\[0.0em]
\end{tabular}
\end{tabular}
\caption{The median of diagnosis ratings for Llama across the stories in the dataset. Left: C$\rightarrow$I, Right: I$\rightarrow$C. Note that for both directions higher rating is better.}
\label{fig:med_diag_comparison_llama}
\end{figure*}

\begin{figure*}[t!]
\begin{tabular}{@{}c@{\hspace{0em}}c@{}}
\begin{tabular}{@{}c@{}}
    \textbf{C$\rightarrow$I} \\[0em]
        \includegraphics[width=0.49\linewidth]{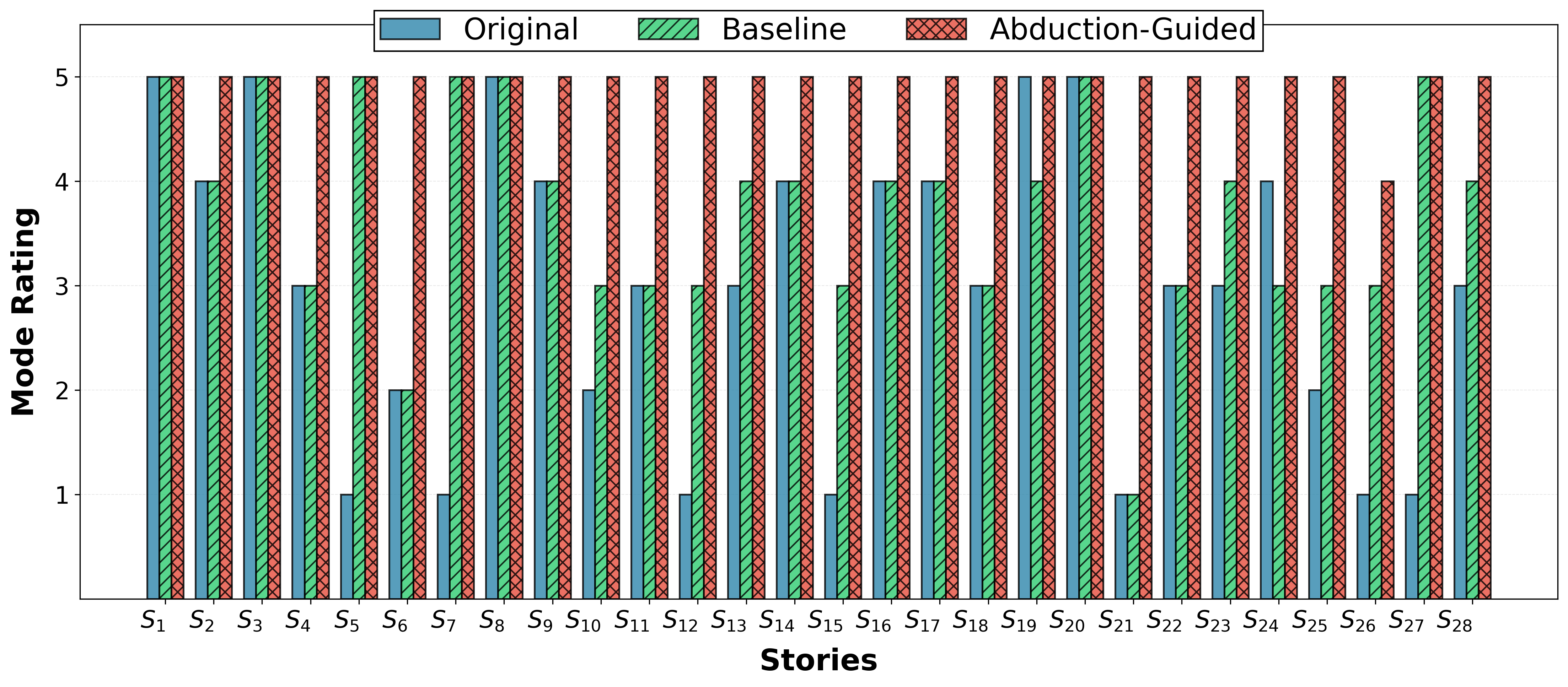}\\[0.0em]
\end{tabular}
&
\begin{tabular}{@{}c@{}}
    \textbf{I$\rightarrow$C} \\[0em]
        \includegraphics[width=0.49\linewidth]{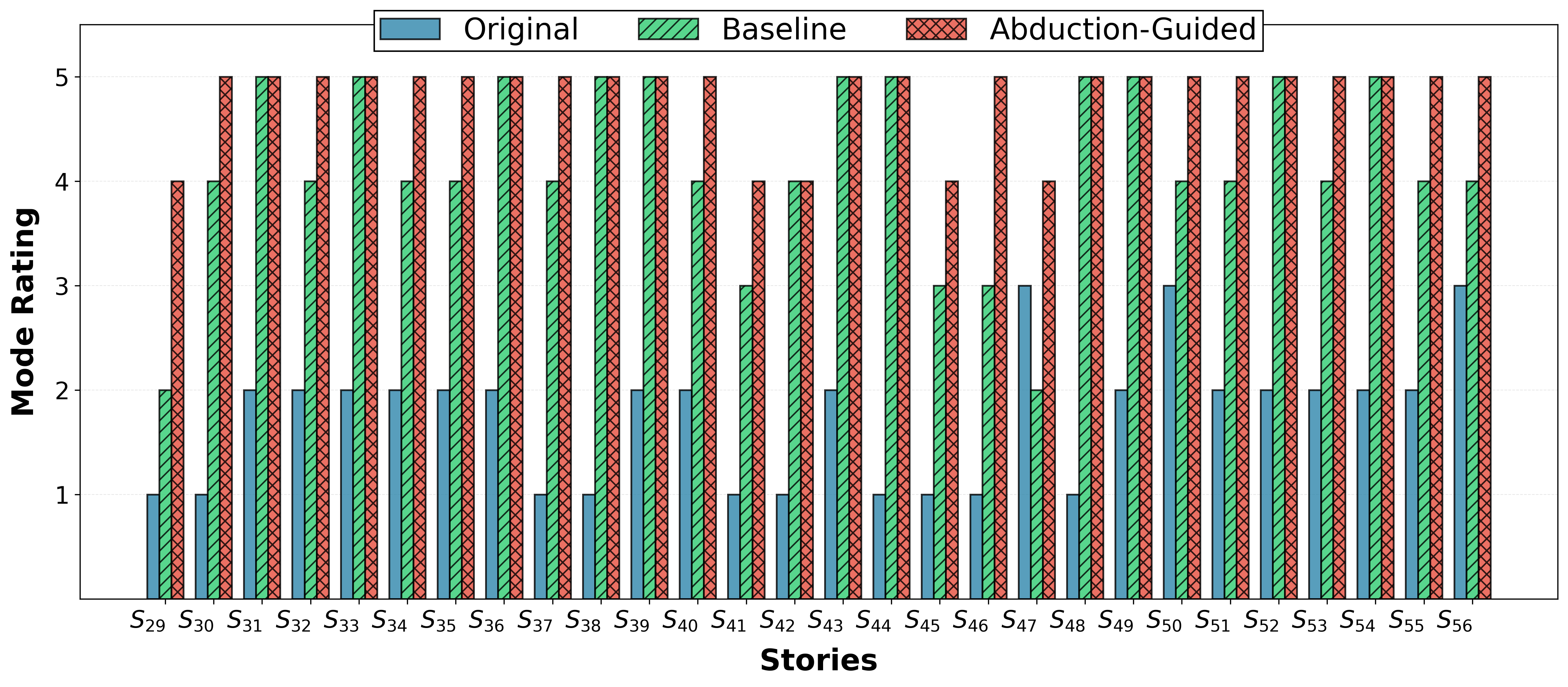}\\[0.0em]
\end{tabular}
\end{tabular}
\caption{The mode of diagnosis ratings for GPT-4o across the stories in the dataset. Left: C$\rightarrow$I, Right: I$\rightarrow$C. Note that for both directions higher rating is better.}
\label{fig:mode_diag_comparison}
\end{figure*}
\begin{figure*}[!t]
\begin{tabular}{@{}c@{\hspace{0em}}c@{}}
\begin{tabular}{@{}c@{}}
    \textbf{C$\rightarrow$I} \\[0em]
        \includegraphics[width=0.49\linewidth]{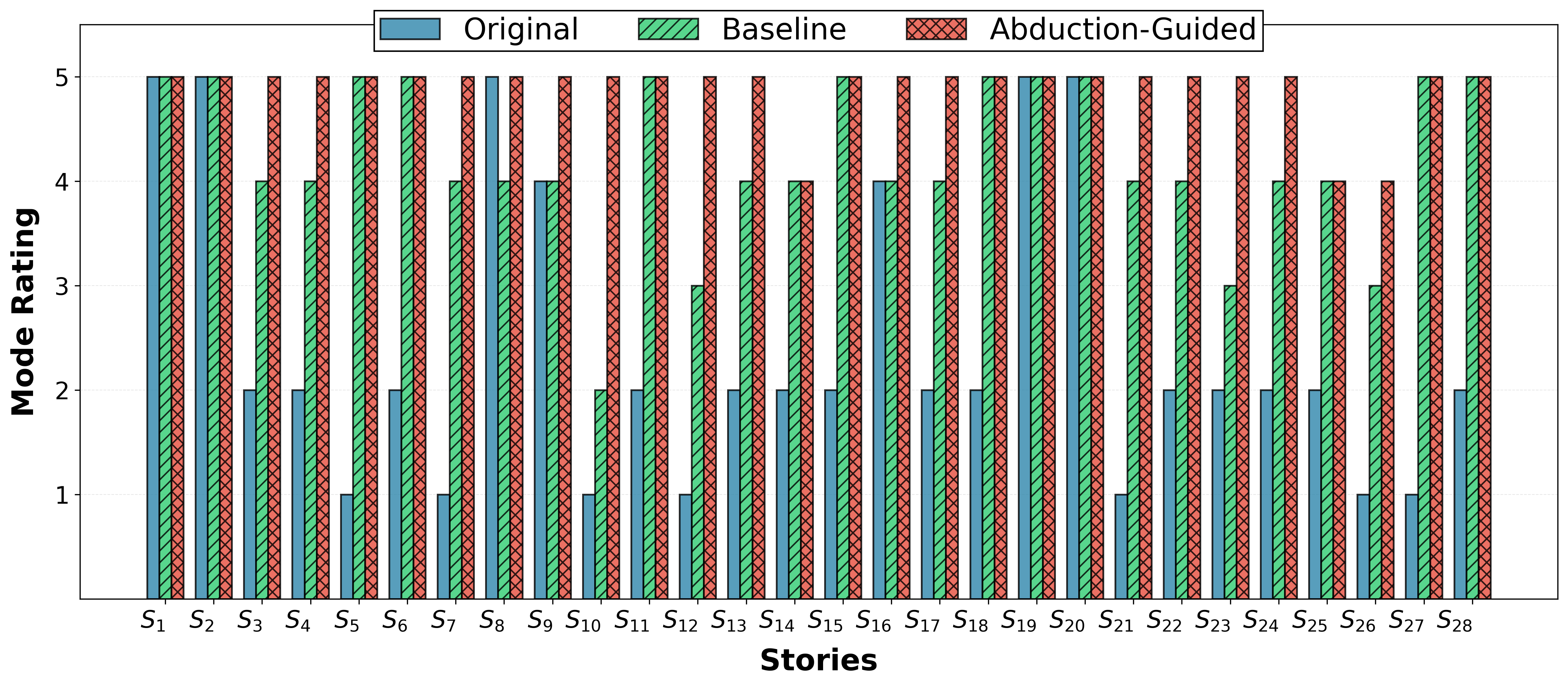}\\[0.0em]
\end{tabular}
&
\begin{tabular}{@{}c@{}}
    \textbf{I$\rightarrow$C} \\[0em]
        \includegraphics[width=0.49\linewidth]{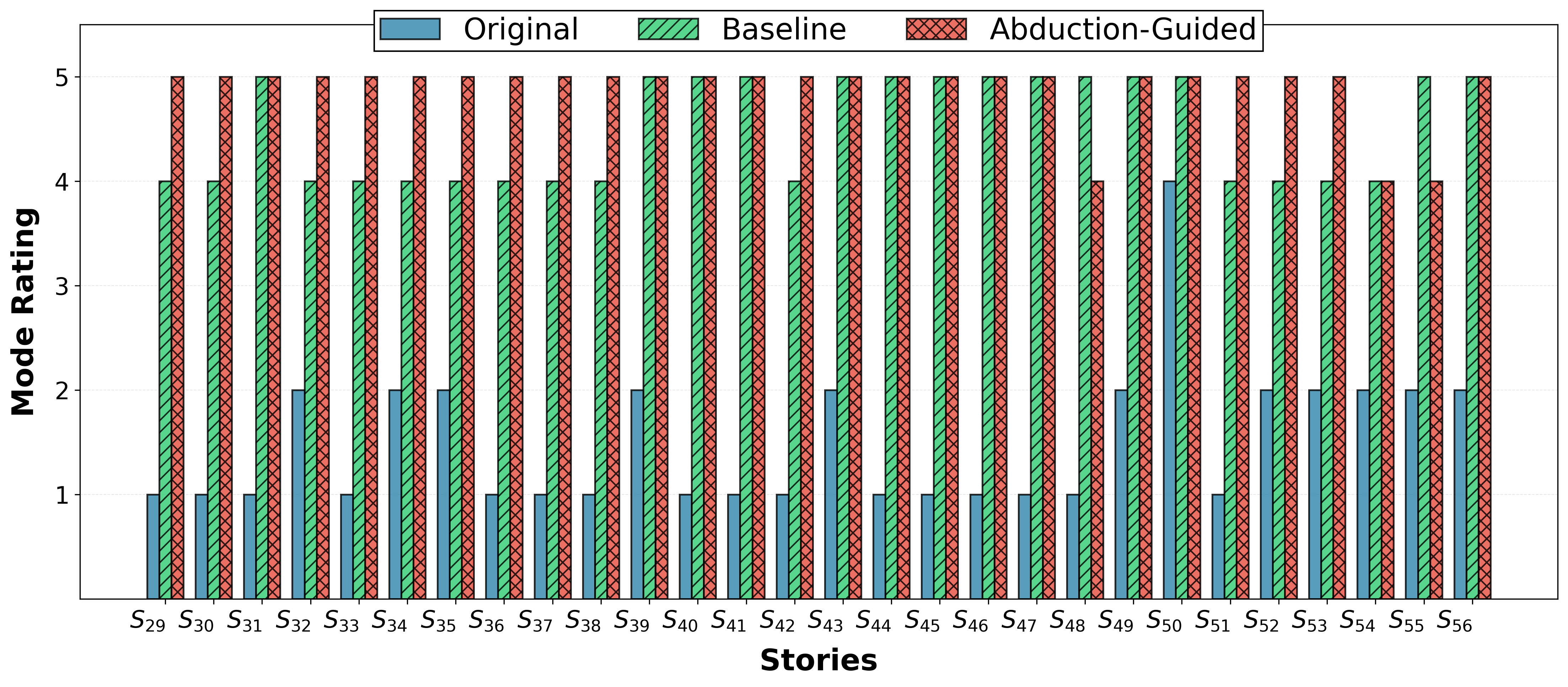}\\[0.0em]
\end{tabular}
\end{tabular}
\caption{The mode of diagnosis ratings for Grok-4 across the stories in the dataset. Left: C$\rightarrow$I, Right: I$\rightarrow$C. Note that for both directions higher rating is better.}
\label{fig:mode_diag_comparison_grok4}
\end{figure*}

\begin{figure*}[!t]
\begin{tabular}{@{}c@{\hspace{0em}}c@{}}
\begin{tabular}{@{}c@{}}
    \textbf{C$\rightarrow$I} \\[0em]
        \includegraphics[width=0.49\linewidth]{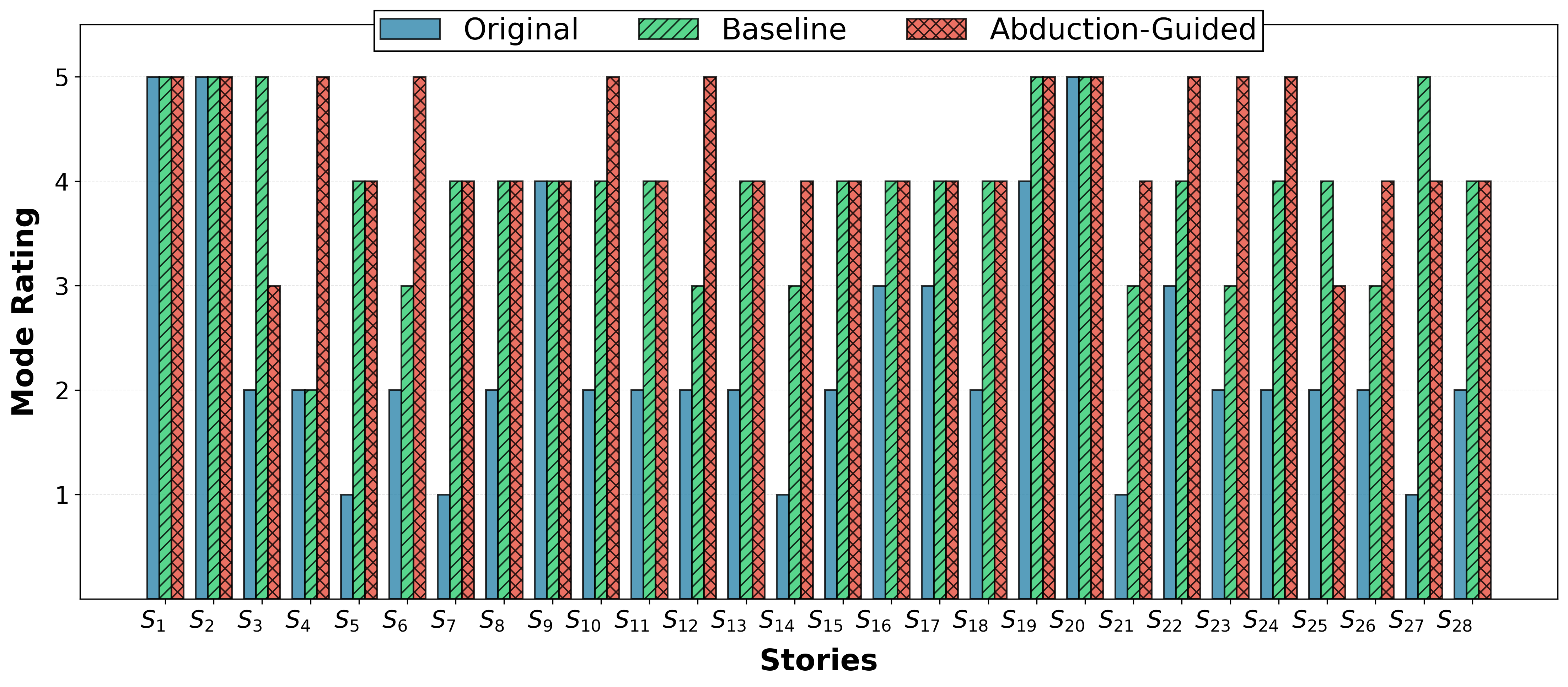}\\[0.0em]
\end{tabular}
&
\begin{tabular}{@{}c@{}}
    \textbf{I$\rightarrow$C} \\[0em]
        \includegraphics[width=0.49\linewidth]{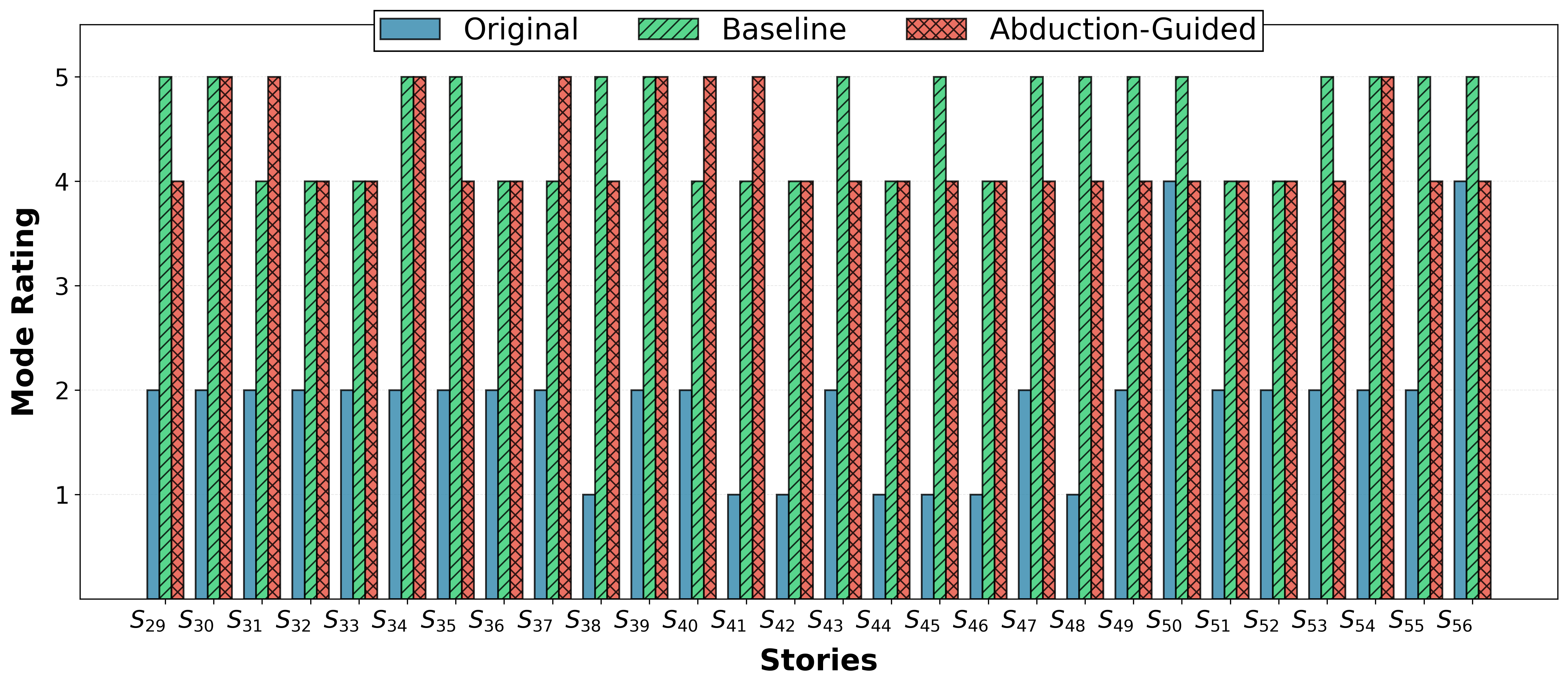}\\[0.0em]
\end{tabular}
\end{tabular}
\caption{The mode of diagnosis ratings for Deepseek-R1 across the stories in the dataset. Left: C$\rightarrow$I, Right: I$\rightarrow$C. Note that for both directions higher rating is better.}
\label{fig:mode_diag_comparison_deepseek}
\end{figure*}

\begin{figure*}[!t]
\begin{tabular}{@{}c@{\hspace{0em}}c@{}}
\begin{tabular}{@{}c@{}}
    \textbf{C$\rightarrow$I} \\[0em]
        \includegraphics[width=0.49\linewidth]{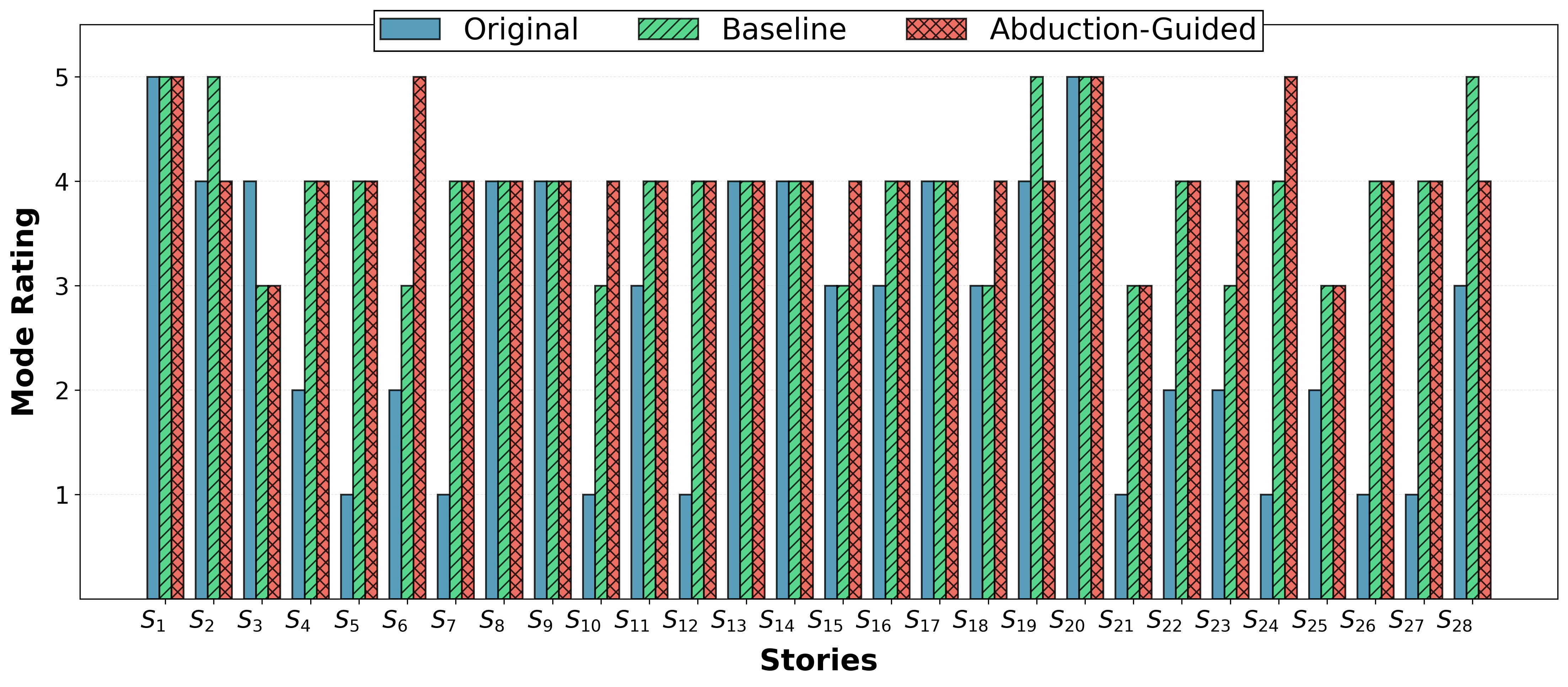}\\[0.0em]
\end{tabular}
&
\begin{tabular}{@{}c@{}}
    \textbf{I$\rightarrow$C} \\[0em]
        \includegraphics[width=0.49\linewidth]{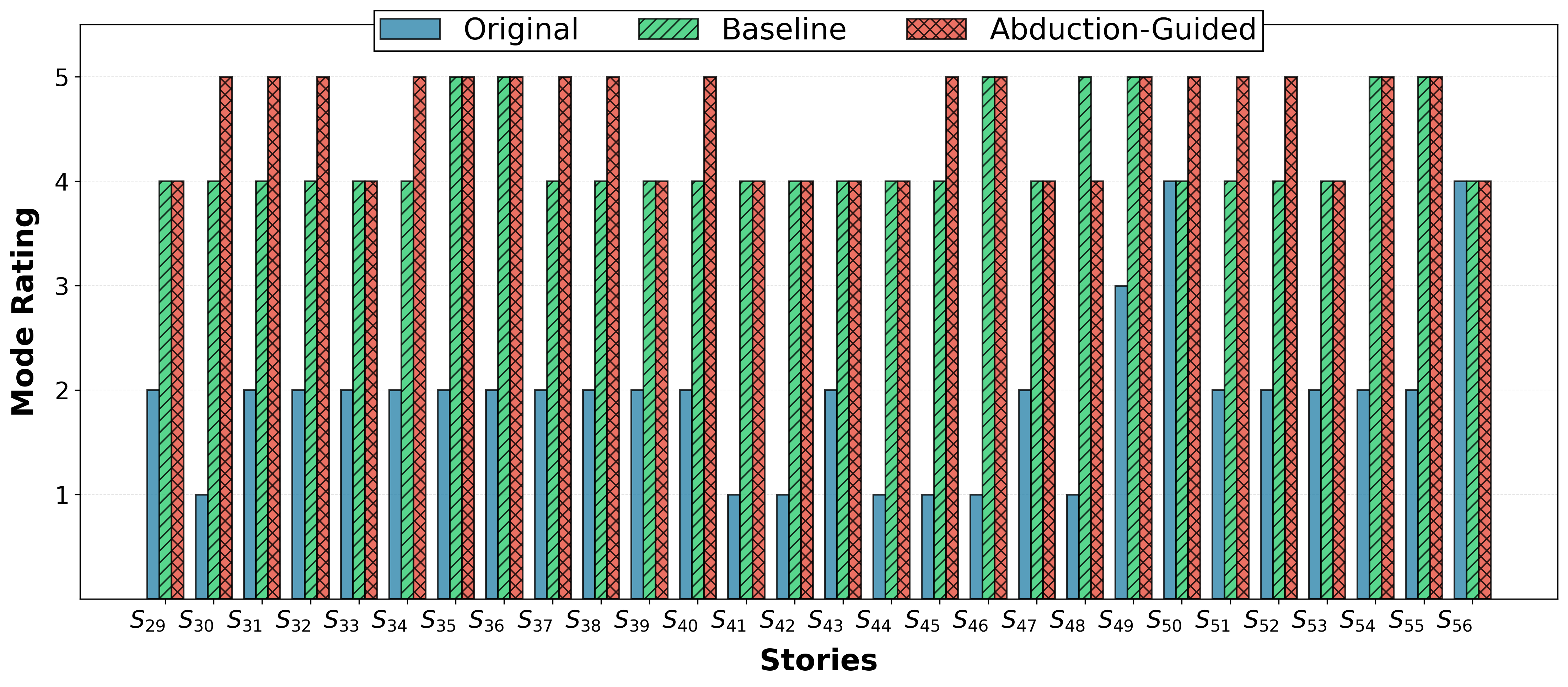}\\[0.0em]
\end{tabular}
\end{tabular}
\caption{The mode of diagnosis ratings for Llama across the stories in the dataset. Left: C$\rightarrow$I, Right: I$\rightarrow$C. Note that for both directions higher rating is better.}
\label{fig:mode_diag_comparison_llama}
\end{figure*}

\begin{figure*}[t!]
\begin{tabular}{@{}c@{\hspace{0em}}c@{}}
\begin{tabular}{@{}c@{}}
    \textbf{C$\rightarrow$I} \\[0em]
        \includegraphics[width=0.49\linewidth]{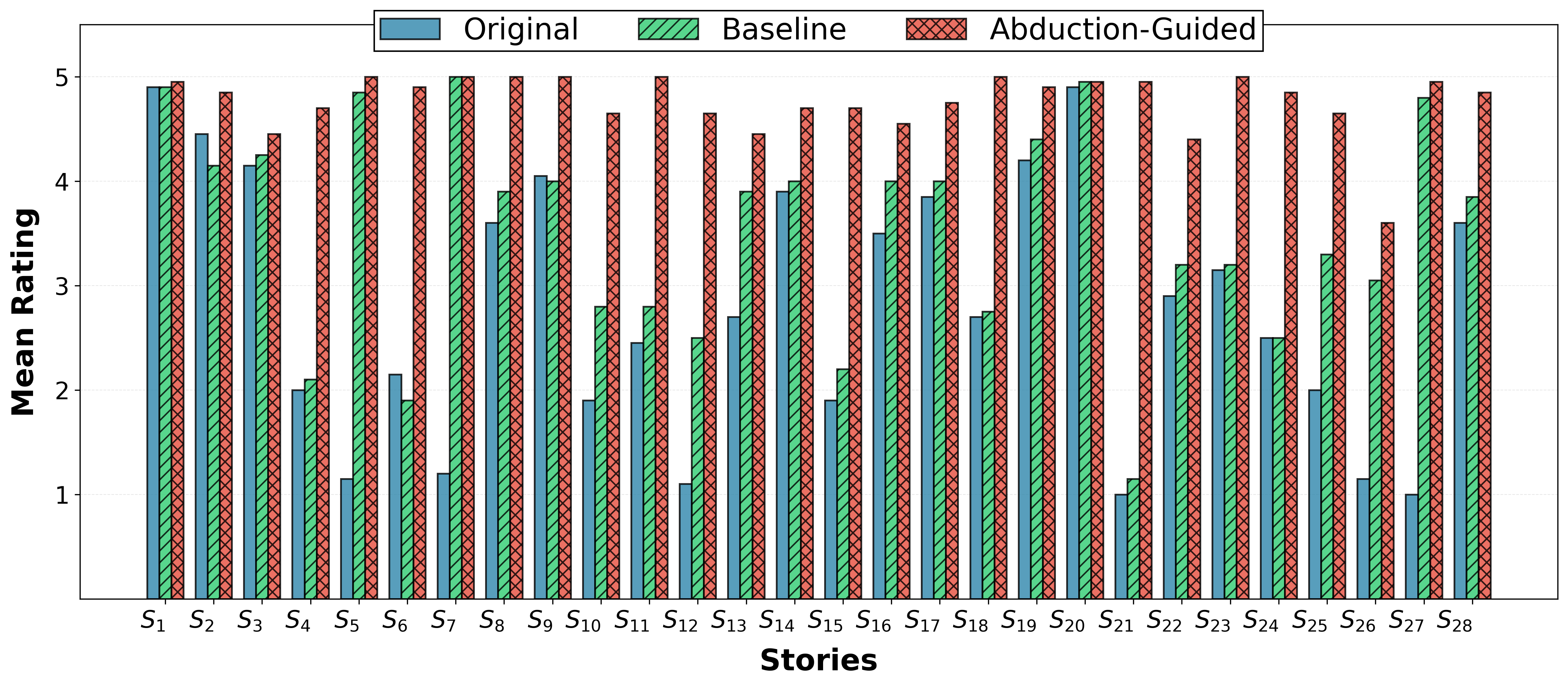}\\[0.0em]
\end{tabular}
&
\begin{tabular}{@{}c@{}}
    \textbf{I$\rightarrow$C} \\[0em]
        \includegraphics[width=0.49\linewidth]{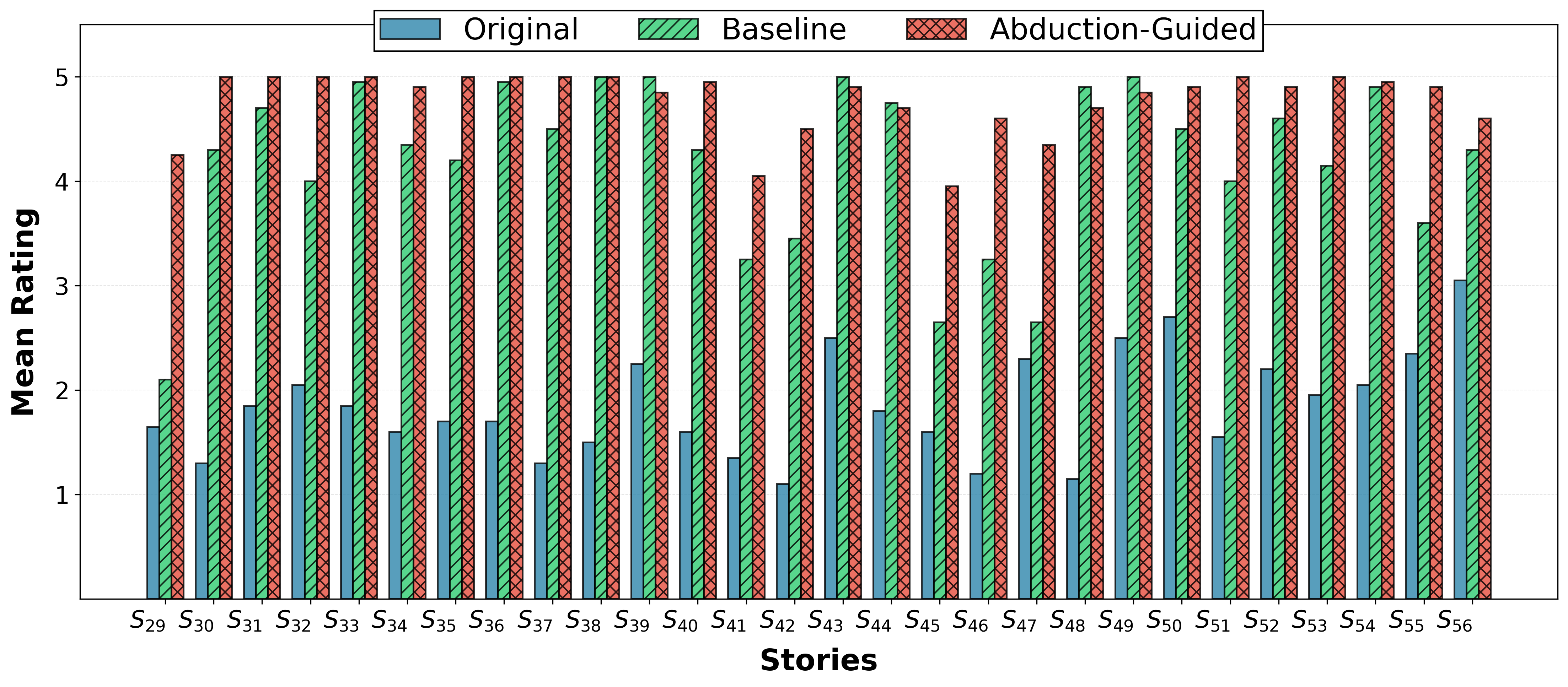}\\[0.0em]
\end{tabular}
\end{tabular}
\caption{The mean of diagnosis ratings for GPT-4o across the stories in the dataset. Left: C$\rightarrow$I, Right: I$\rightarrow$C. Note that for both directions higher rating is better.}
\label{fig:mean_diag_comparison}
\end{figure*}
\begin{figure*}[!t]
\begin{tabular}{@{}c@{\hspace{0em}}c@{}}
\begin{tabular}{@{}c@{}}
    \textbf{C$\rightarrow$I} \\[0em]
        \includegraphics[width=0.49\linewidth]{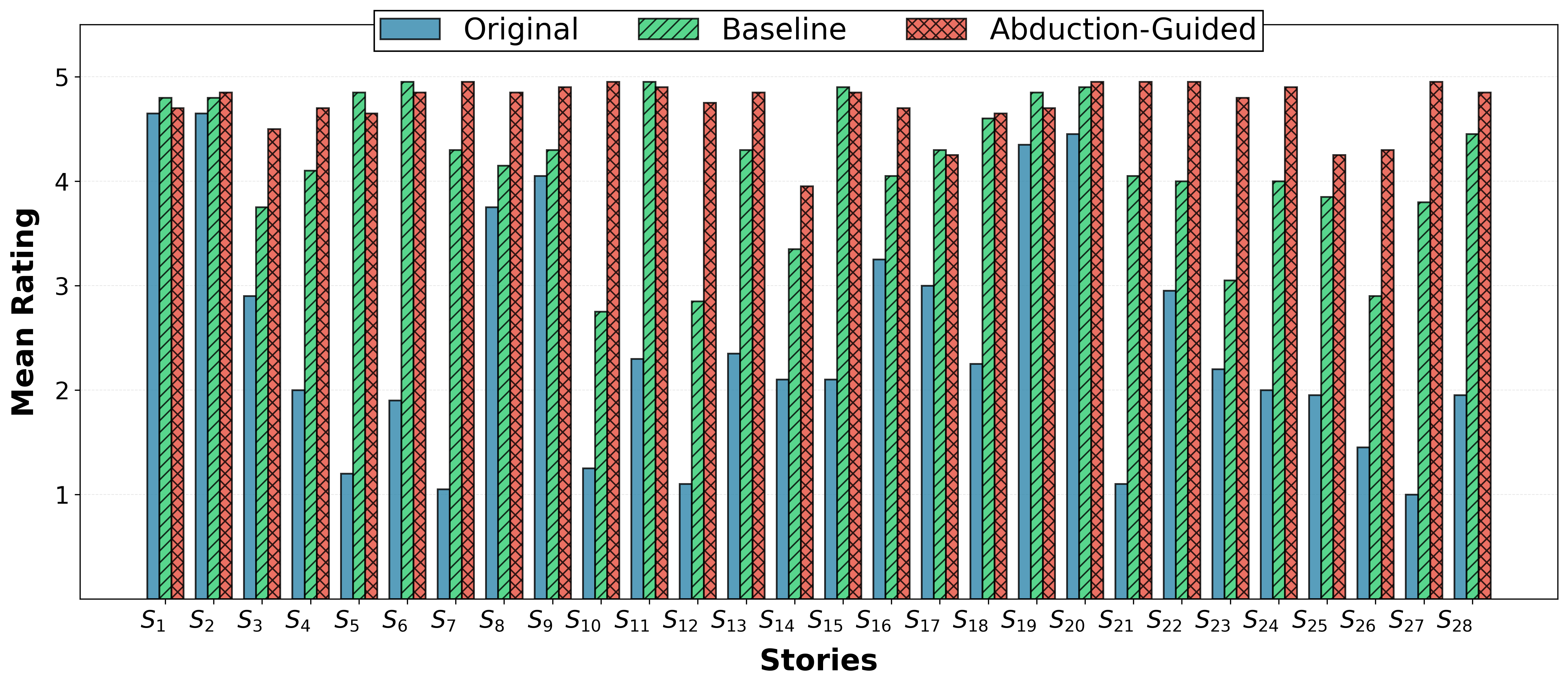}\\[0.0em]
\end{tabular}
&
\begin{tabular}{@{}c@{}}
    \textbf{I$\rightarrow$C} \\[0em]
        \includegraphics[width=0.49\linewidth]{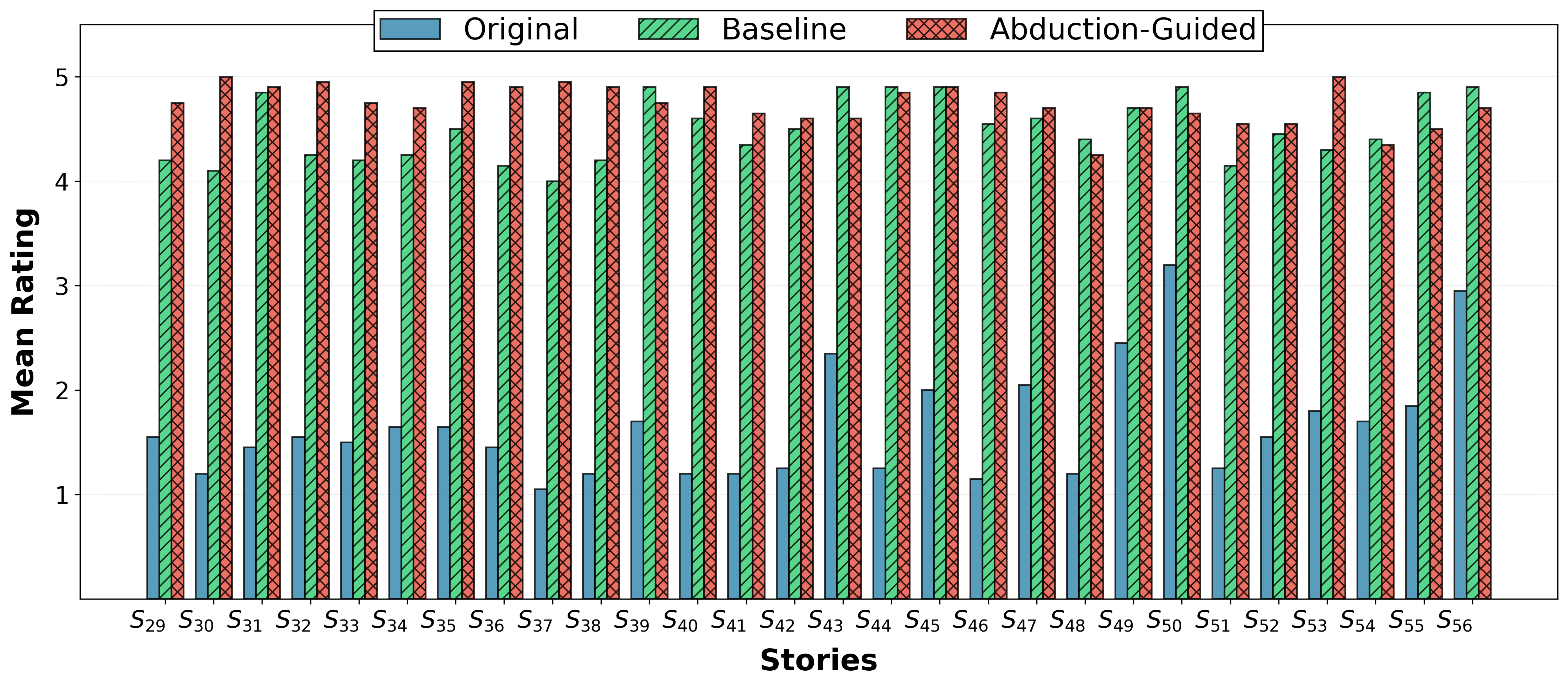}\\[0.0em]
\end{tabular}
\end{tabular}
\caption{The mean of diagnosis ratings for Grok-4 across the stories in the dataset. Left: C$\rightarrow$I, Right: I$\rightarrow$C. Note that for both directions higher rating is better.}
\label{fig:mean_diag_comparison_grok4}
\end{figure*}

\begin{figure*}[!t]
\begin{tabular}{@{}c@{\hspace{0em}}c@{}}
\begin{tabular}{@{}c@{}}
    \textbf{C$\rightarrow$I} \\[0em]
        \includegraphics[width=0.49\linewidth]{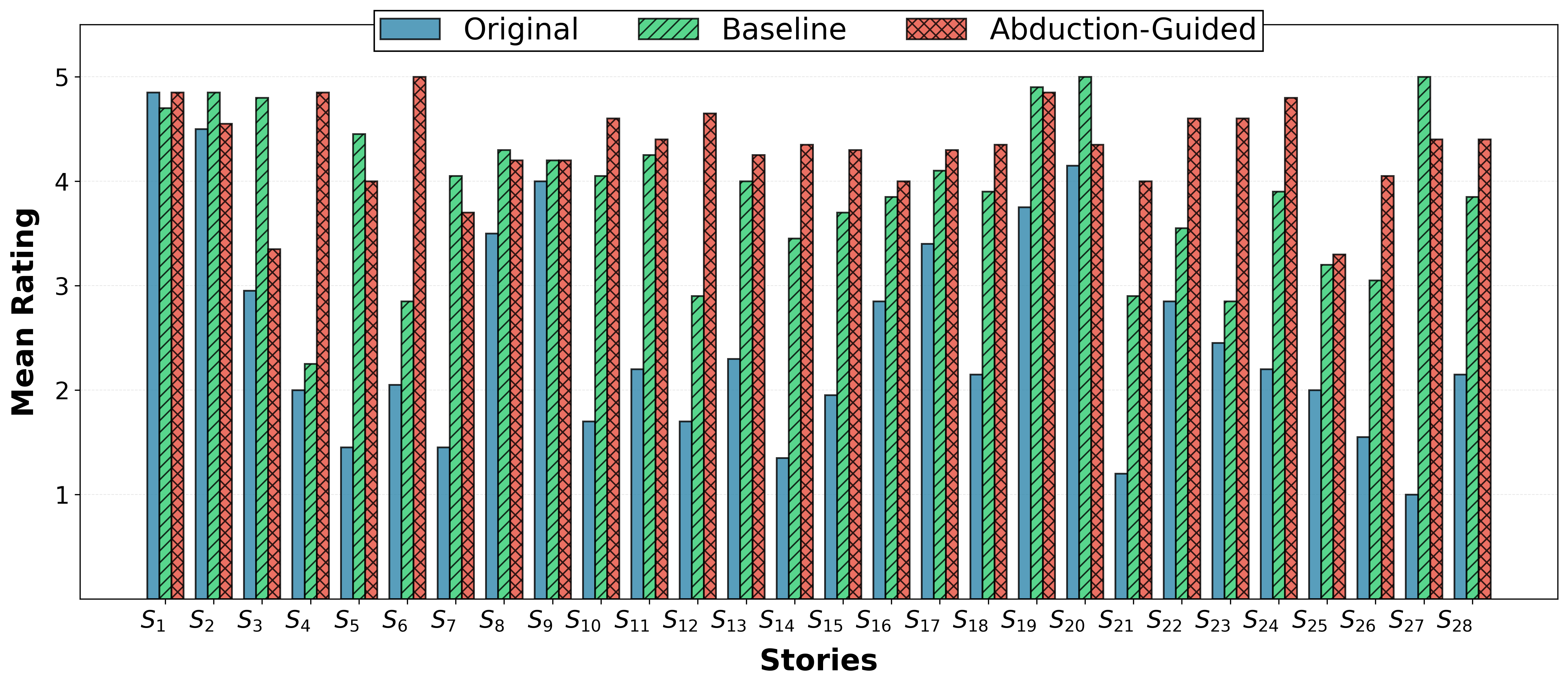}\\[0.0em]
\end{tabular}
&
\begin{tabular}{@{}c@{}}
    \textbf{I$\rightarrow$C} \\[0em]
        \includegraphics[width=0.49\linewidth]{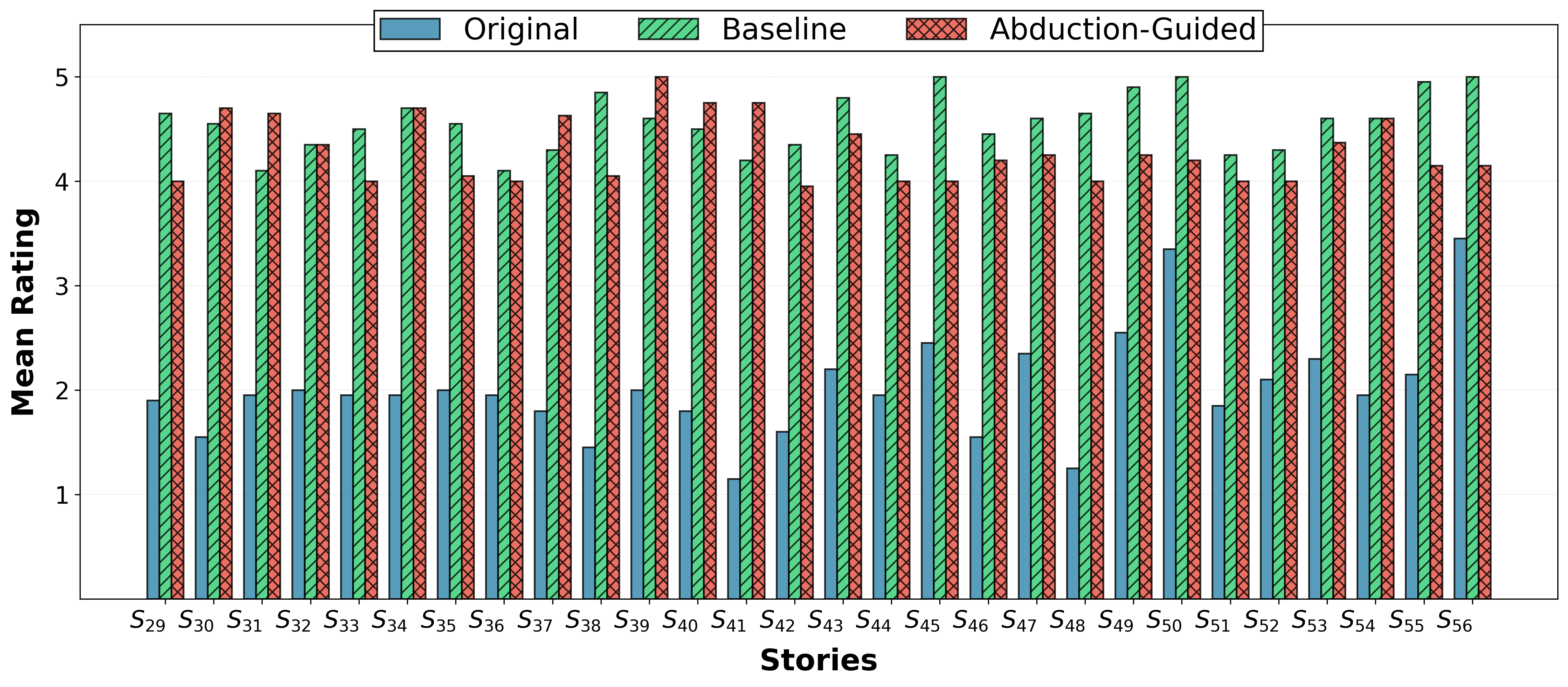}\\[0.0em]
\end{tabular}
\end{tabular}
\caption{The mean of diagnosis ratings for Deepseek-R1 across the stories in the dataset. Left: C$\rightarrow$I, Right: I$\rightarrow$C. Note that for both directions higher rating is better.}
\label{fig:mean_diag_comparison_deepseek}
\end{figure*}

\begin{figure*}[!t]
\begin{tabular}{@{}c@{\hspace{0em}}c@{}}
\begin{tabular}{@{}c@{}}
    \textbf{C$\rightarrow$I} \\[0em]
        \includegraphics[width=0.49\linewidth]{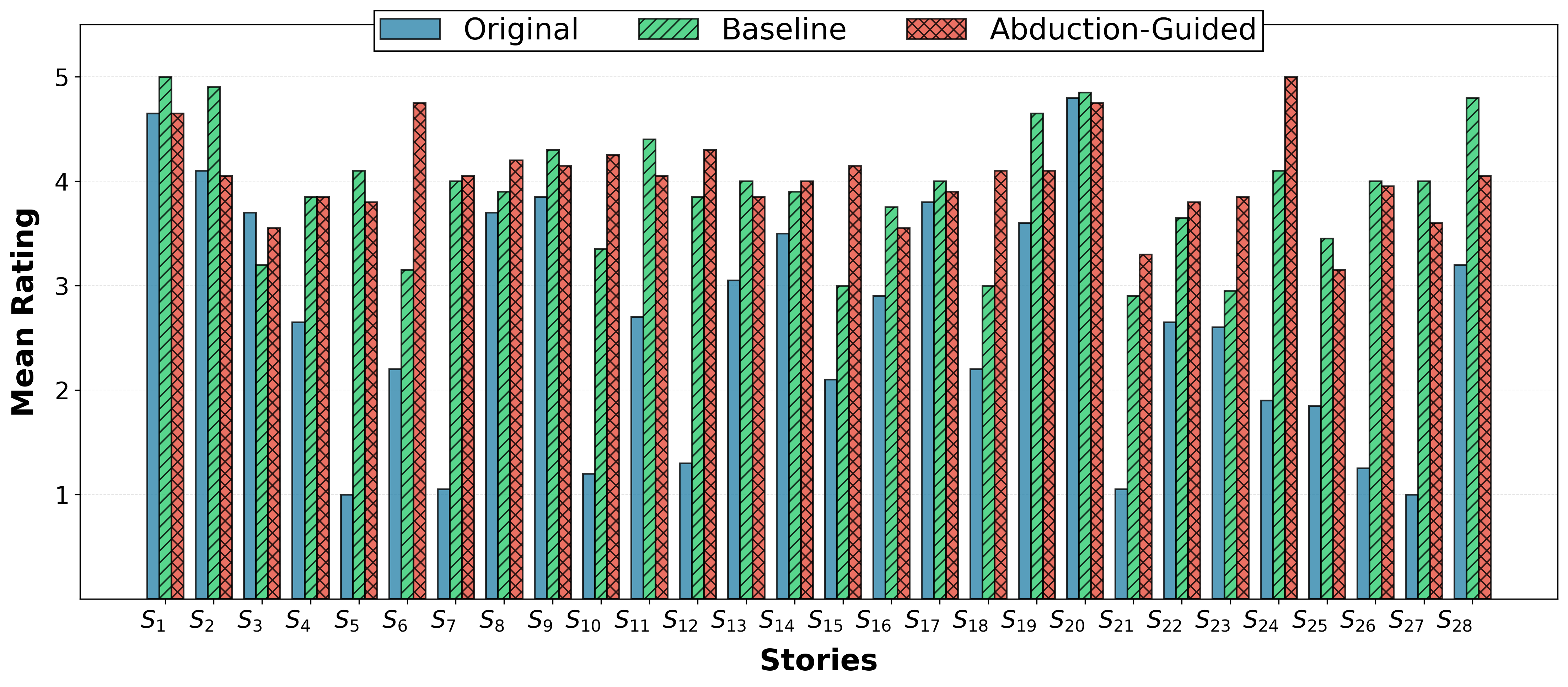}\\[0.0em]
\end{tabular}
&
\begin{tabular}{@{}c@{}}
    \textbf{I$\rightarrow$C} \\[0em]
        \includegraphics[width=0.49\linewidth]{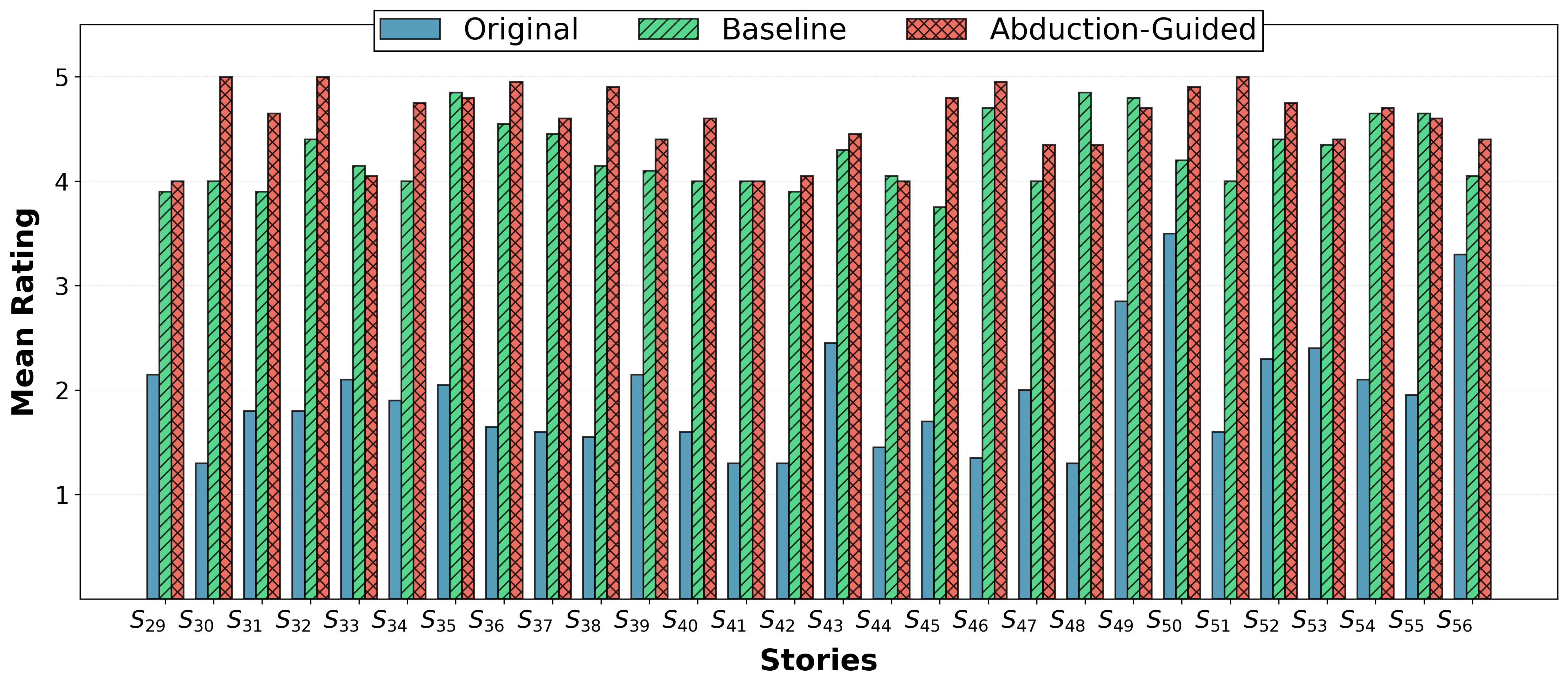}\\[0.0em]
\end{tabular}
\end{tabular}
\caption{The mean of diagnosis ratings for Llama across the stories in the dataset. Left: C$\rightarrow$I, Right: I$\rightarrow$C. Note that for both directions higher rating is better.}
\label{fig:mean_diag_comparison_llama}
\end{figure*}

\begin{figure*}[t]
    \centering
    \fbox{
    \begin{minipage}{0.98\textwidth}
        \small
\textbf{Original (Collectivistic):}
\begin{quote}
\small The group of [women] from the community are warriors — women with great courage, enthusiasm, and a spirit of service, said Fernando, a Convoy of Hope agriculture technician in Nicaragua. We have often seen them give up their time to help other people. This work that we have done supporting the Agriculture program has brought us other experiences, the Warrior Women said. Agriculture is not just a job for men. Now we feel empowered to produce our own food. We are grateful to all … of Convoy of Hope, especially to the facilitators. We believe that they have given us the opportunity to learn, and that is why we now have the experiences to continue improving our lives, our family, and our community.
\end{quote}

\textbf{Baseline (Zero-shot LLM): }{\color{red} $\times$ Maintains collectivistic framing}
\begin{quote}
\small ``Andria and her group are true warriors,'' praises Fernando, a Convoy of Hope agriculture technician. ``Their spirit of service is remarkable, and their contribution invaluable. Through this journey, Andria has gained more than agricultural skills. ``Agriculture isn't just for men,'' she asserts. ``We've learned to produce our own food and feel empowered by it.'' \textit{Grateful for the support from Convoy of Hope,} Andria is optimistic about the future. \textit{``The opportunity to learn has enriched our lives, our families, and our community,''} she reflects.''
\end{quote}

\textbf{Our Method: }{\color{blue} $\checkmark$ Individualistic Shift}
\begin{quote}
\small ``\textit{Each woman is a warrior in her own right — an individual with extraordinary courage, enthusiasm, and a personal spirit of service,} said Fernando, a Convoy of Hope agriculture technician in Nicaragua. We have often seen them give up their time to help other people.
\textit{This journey I have undertaken in supporting the Agriculture program} has brought \textit{me} other valuable experiences,'' one of the Warrior Women said. ``Agriculture is not just a job for men. Now I feel empowered to produce my own food.
\textit{I am proud of my own determination and efforts. It's my drive and ambition that have allowed me to learn} and gain the experiences needed to improve my life and the lives of those around me.''
\end{quote}

    \end{minipage}
    }
    \caption{Abduction-guided LLM transformation successfully shifts narrative from C$\rightarrow$I.}
    \label{fig:c_i_transform}
\end{figure*}

\end{document}